\documentclass[journal]{IEEEtran}

\usepackage[utf8]{inputenc}
\usepackage[english]{babel}
\usepackage{graphicx}
\graphicspath{{./figures/}}
\usepackage[caption=false,font=footnotesize]{subfig}

\usepackage{amsmath,amsthm,amssymb}
\usepackage{mathrsfs}
\usepackage{mathtools}
\mathtoolsset{showonlyrefs}
\usepackage{mathrsfs}
\usepackage{array}
\usepackage{xcolor}
\usepackage{algorithm}
\usepackage{algpseudocode}
\usepackage{csquotes}
\usepackage{comment}
\usepackage{bbm}

\newtheorem{theorem}{Theorem}
\newtheorem{definition}{Definition}

\newtheorem{remark}{Remark}

\newtheorem{proposition}{Proposition}

\newtheorem{example}{Example}

\newcommand{\R}{\mathbb{R}}

\newcommand{\tr}{^T}

\newcommand{\mc}{\mathcal}

\newcommand{\Deltat}{\Delta t}
\newcommand{\xk}{x^{(k)}}

\newcommand{\alphak}{\alpha^{(k)}}
\newcommand{\xkn}{x^{(k+n)}}
\newcommand{\ukn}{u^{(k+n)}}
\newcommand{\hatukn}{\hat u^{(k+n)}}

\newcommand{\alphakn}{\alpha^{(k+n)}}

\newcommand{\GammaMIQP}{\Gamma_\text{MIQP}}
\newcommand{\GammabarMIQP}{\bar\Gamma_\text{MIQP}}
\newcommand{\GammaQP}{\Gamma_\text{QP}}
\newcommand{\LMIQP}{L_\text{MIQP}}
\newcommand{\LQP}{L_\text{QP}}
\newcommand{\x}[1]{x^{(#1)}}

\DeclareMathOperator*{\minimize}{minimize~}
\DeclareMathOperator*{\subjto}{subject\,to~}

\newcolumntype{L}[1]{>{\raggedright\let\newline\\\arraybackslash\hspace{0pt}}m{#1}}
\newcolumntype{C}[1]{>{\centering\let\newline\\\arraybackslash\hspace{0pt}}m{#1}}
\newcolumntype{R}[1]{>{\raggedleft\let\newline\\\arraybackslash\hspace{0pt}}m{#1}}

\begin{document}

\title{A Resilient and Energy-Aware Task Allocation Framework for Heterogeneous Multi-Robot Systems}

\author{Gennaro Notomista$^{1}$, Siddharth Mayya$^{2}$, Yousef Emam$^{1}$,\\Christopher Kroninger$^{3}$, Addison Bohannon$^{3}$, Seth Hutchinson$^{1}$, Magnus Egerstedt$^{1}$
\thanks{\textcopyright 2021 IEEE.  Personal use of this material is permitted.  Permission from IEEE must be obtained for all other uses, in any current or future media, including reprinting/republishing this material for advertising or promotional purposes, creating new collective works, for resale or redistribution to servers or lists, or reuse of any copyrighted component of this work in other works.}
\thanks{This research was sponsored by the Army Research Lab through ARL DCIST CRA W911NF-17-2-0181.}%
\thanks{$^{1}$G. Notomista, Y. Emam, S. Hutchinson, and M. Egerstedt are with the Institute for Robotics and Intelligent Machines, Georgia
Institute of Technology, Atlanta, GA 30332, USA,
{\tt\small\{g.notomista, emamy, seth, magnus\}@gatech.edu.}}%
\thanks{$^{2}$S. Mayya is with the GRASP Laboratory, University of Pennsylvania, Philadelphia, PA, USA {\tt\small mayya@seas.upenn.edu}}%
\thanks{$^{3}$C. Kroninger and A. Bohannon are with the Combat Capabilities Development Command, Army Research Laboratory (CCDC ARL) {\tt\small\{christopher.m.kroninger.civ,}\newline{\tt\small addison.w.bohannon.civ\}@mail.mil}}%
}

\maketitle

\begin{abstract}
In the context of heterogeneous multi-robot teams deployed for executing multiple tasks, this paper develops an energy-aware framework for allocating tasks to robots in an online fashion. With a primary focus on long-duration autonomy applications, we opt for a survivability-focused approach. Towards this end, the task prioritization and execution---through which the allocation of tasks to robots is effectively realized---are encoded as constraints within an optimization problem aimed at minimizing the energy consumed by the robots at each point in time. In this context, an allocation is interpreted as a prioritization of a task over all others by each of the robots. Furthermore, we present a novel framework to represent the heterogeneous capabilities of the robots, by distinguishing between the features available on the robots, and the capabilities enabled by these features. By embedding these descriptions within the optimization problem, we make the framework resilient to situations where environmental conditions make certain features unsuitable to support a capability and when component failures on the robots occur. We demonstrate the efficacy and resilience of the proposed approach in a variety of use-case scenarios, consisting of simulations and real robot experiments.    
\end{abstract}

\begin{IEEEkeywords}
Task Planning, Path Planning for Multiple Mobile Robots or Agents, Failure Detection and Recovery, Energy and Environment-Aware Automation, Multi-Robot Systems, Robust/Adaptive Control of Robotic Systems
\end{IEEEkeywords}

\IEEEpeerreviewmaketitle
\section{Introduction}
\label{sec:intro}

\IEEEPARstart{M}{ulti}-robot task allocation (MRTA) is an active research topic given the increasing deployment of multi-robot systems in dynamic and partially unknown out-of-laboratory environments (see, e.g., \cite{taxonomy,korsah2013comprehensive} and references therein). Often, the design of MRTA algorithms is tailored around particular challenges that the multi-robot team is expected to face in the environment. For instance, many envisioned applications require robots with limited energy resources to operate effectively for long periods of time, necessitating the development of \emph{survivability}-focused energy-aware algorithms for task execution as well as allocation \cite{egerstedt2018robot,notomista2019optimal}. Similarly, robot \emph{heterogeneity} has received explicit focus within the MRTA literature, as teams equipped with different types of sensors, actuators, and communication devices can enable the execution of a wider range of tasks~\cite{parker1994heterogeneous,iocchi2003distributed,prorok2017impact}. \par 
Heterogeneity can also contribute favorably to another desirable property of a MRTA framework: \emph{resilience}, typically interpreted as the ability of the allocation algorithm to react to component failures on the robots, varying environmental conditions, and other non-idealities in the operating conditions~\cite{ramachandran2019resilience}. Consider an example scenario where a heterogeneous multi-robot system consisting of ground and aerial mobile robots, is tasked with carrying objects to specified locations in the environment. Unexpected weather conditions might lead to high speed winds in the environment, which might prevent aerial robots from making further progress towards the accomplishment of their goal. In such a case, the heterogeneity in the capabilities of the robots could be leveraged through a dynamic re-allocation of the transport task to ground robots.\par

It should be noted that, in light of such a scenario, the multi-robot task \emph{allocation} problem can be considered as being inextricably linked to the \emph{execution} of the tasks by the robots. This is especially true when considering the deployment of survivability-focused multi-robot systems over long time horizons, where evolving or newly detected environmental phenomena can affect the task allocation. \par 

This paper presents a dynamic task allocation and execution framework for multi-robot systems which explicitly accounts for the aforementioned survivability and heterogeneity considerations while being demonstrably resilient to robot failures and changes in environmental conditions. To encode heterogeneity, we propose a novel framework for representing the suitability that robots have for different tasks. This is done by explicitly considering the \emph{capabilities} required to perform the tasks (e.g., flight or high speed) as well as the \emph{features} available on the robots (e.g., a specific type of sensors, actuators, or communication equipment) which support these capabilities. We leverage this representation in a constraint-based optimization framework whose solution at each point in time yields (i) a dynamic allocation of tasks to robots through a prioritization scheme, and (ii) control inputs to each robot which ensure the execution of the tasks in accordance with the optimized priorities \cite{notomista2019optimal}.

Existing task allocation techniques typically define both robots and tasks in terms of the capabilities available on the robots and required to perform the tasks \cite{prorok2017impact,GerkeyB.P.2003MtaA}. In contrast, our approach distinguishes between the features available on the robots and the capabilities that these features enable. We demonstrate that this explicit representation contributes to the resilience of the proposed dynamic task allocation method, leveraging the fact that multiple bundles of robot features can satisfy the same capability. Consequently, dynamic readjustments of Robot-to-Feature and Feature-to-Capability mappings can enhance the resilience of the system by capturing scenarios in which (i) environmental conditions make a certain feature more suitable to support a capability, and (ii) component failures on robots occur, affecting the available features. The pertinent question then becomes: \emph{how can we design a survivability-focused dynamic allocation paradigm based on these descriptions of heterogeneity (at the feature level as well as the capability level) with demonstrably resilient operations?} \par 
Following preliminary work in \cite{notomista2019constraint,notomista2019optimal,emam2020adaptive} on adaptive and minimum-energy task execution and allocation, in this paper, we opt for a \emph{constraint-based} approach, where the execution and prioritization of tasks are encoded as constraints within an optimization problem. Such a formulation has demonstrated both the ability of accounting for the energy limitation that robots have while executing tasks \cite{notomista2018persistification,notomista2020persistification,fouad2020energy}, and a higher flexibility and robustness in scenarios where the operating conditions of the robots are only partially known or may change \cite{emam2020adaptive}---events which are especially likely when considering long-duration autonomy applications. Since energy considerations are of paramount importance in our framework, the execution of $n_t$ different tasks by a robot can be posed as the following energy-minimization problem:
\begin{align} \label{eqn:te_1}
\minimize_{u} & \|u\|^2 \\
\subjto & c_{\text{task}_j}(x,u) \geq 0, \quad\forall j \in \{1,\ldots,n_t\},
\end{align}
where $x$ is the current state of the robot, $u$ is the control effort expended by it, and $c_{\text{task}_j}(x,u) \geq 0$ denotes a constraint function which enforces the execution of task $j$. The feasibility of such an optimization problem is ensured by the introduction of a slack variable $\delta\in\R^{n_t}$:
\begin{align}\label{eqn:te_2}
\minimize_{u,\delta} & \|u\|^2 + \|\delta\|^2 \\
\subjto & c_{\text{task}_j}(x,u) \geq -\delta_{j}, \quad\forall j \in \{1,\ldots,n_t\},
\end{align}
where $\delta = [\delta_{1},\delta_{2},\ldots,\delta_{n_t}]^T$ is a vector with positive components representing the extent to which the robot can violate the constraints corresponding to each of the tasks.

The applicability of this framework for dynamic task allocation via prioritization is enabled by the observation that relative constraints among the components of $\delta$ can allow a robot to perform one task more effectively than others. For instance, if $\delta_{m} \ll \delta_{n}, \forall n \neq m$, then the robot will execute task $m$ with priority higher than all the other tasks: this represents an \emph{allocation via prioritization}. Such a prioritization can be then encoded via an additional constraint $K\delta \leq 0$ in \eqref{eqn:te_2}, where the matrix $K$ encodes the relative inequalities among the slack variables. \par 
Within the above described formulation, the allocation problem then consists of designing the matrix $K$ for each robot such that the heterogeneity of the robots---intended as their different ability of performing different tasks---are appropriately accounted for. To this end, we propose a modification of the minimum energy optimization problem presented in \eqref{eqn:te_2} where priority matrices $K$ are automatically generated. Moreover, by means of an additional constraint, the optimization problem ensures that the minimum amount of capabilities required for the successful execution of each task is met by the allocation. \par 
This formulation yields a mixed-integer quadratic program (MIQP) which not only generates the task allocation of the team (encoded via the prioritization matrices $K$) but also the control inputs $u$ which the robots can use to execute the tasks. Since MIQPs can be computationally intensive to solve, we further present a mixed centralized/decentralized computational architecture which allows a central coordinator to transmit the task priorities to each robot. The robots can then solve the simpler convex quadratic program described in \eqref{eqn:te_2} (with the additional constraint $K\delta \leq 0$) to generate their control inputs in real time. \par 
We demonstrate that non-idealities such as environmental disturbances and/or component failures on the robots can be effectively accounted for in our framework, enabling, this way, resilient task allocation behaviors. It is informative to know how, also in nature, such behaviors emerge from the concepts of survivability and heterogeneity. In fact, these concepts play a central role in ecological studies as well, as highlighted by Bridle and van Rensburg in \cite{bridle2020discovering}:
\begin{displayquote}
\itshape
For some groups of organisms, we can now integrate genomic data with environmental and demographic data to test the extent to which ecological resilience depends on evolutionary adaptation. Such data will allow researchers to estimate when and where biodiversity within a species has the power to rescue ecological communities from collapse due to climate change and habitat loss.
\end{displayquote}
Drawing an analogy with the task allocation framework we present in this paper, the features of the robots (\emph{genomic data}) and the resulting heterogeneity (\emph{biodiversity}) are leveraged to introduce a degree of resilience (\emph{ecological resilience}) into the framework, which results in a natural adaptation of the multi-robot system to failures (\emph{collapse}) due to the dynamic environments in which it operates (\emph{climate change and habitat loss}).

The remainder of the paper is organized as follows. Section \ref{sec:lit_review_prob_form} introduces the problem formulation and places it within the context of existing literature. Section \ref{sec:encRobotHeter} develops a novel framework for encoding robot heterogeneity. In Section \ref{sec:te}, we present the main constraint-based minimum energy task allocation paradigm, and demonstrate its resilient capabilities in two distinct failure scenarios. Section \ref{sec:mixedc-dec} touches upon the performance guarantees of the developed task allocation paradigm and highlights a mixed centralized/decentralized framework to enable the task allocation and execution. In Section \ref{sec:exp}, we present example use-case scenarios highlighting the resilient allocation and execution of multiple tasks. Section \ref{sec:conc} concludes the paper.

\section{Problem Formulation and Related Work}
\label{sec:lit_review_prob_form}

\subsection{Problem Formulation} 
\label{subsec:problem_formulation}
Consider a team of $n_r$ heterogeneous robots which are deployed in an environment and required to execute $n_t$ tasks. Each robot is endowed with a subset of $n_f$ available features (such as camera, LIDAR, and wheels). These features allow the robots to exhibit  a subset of $n_c$ capabilities (such as flight and navigation). Certain capabilities can be achieved by multiple combinations of feature bundles, whereas tasks require a given set of capabilities in order to be executed. The successful execution of each task is conditioned upon a minimum number of robots with specific capabilities being allocated to it. In this paper, we consider \emph{extended set-based tasks} \cite{notomista2020set}, which include tasks whose execution can be encoded as the minimization of a cost function. \par 
Given the above problem setup, the paper then concerns itself with (i) allocating tasks among the robots such that the minimum requirements for each task are met, and (ii) executing the tasks by synthesizing an appropriate control input for the robots. Both these objectives are met while minimizing the control effort expended by the robots. Additionally, we show how the resulting task allocation and execution framework exhibits resilience properties against varying environment conditions and failures on the robots. 

\subsection{Related Work}
\label{subsec:lit_review}
In this section, we briefly review the relevant literature on MRTA, focusing specifically on the scenarios where tasks might require coordination among multiple robots. For a comprehensive survey on a broader class of task allocation problems, see \cite{taxonomy,korsah2013comprehensive,nunes2017taxonomy} and references within. \par 
In~\cite{parker1994heterogeneous}, the authors developed a framework for assigning heterogeneous robots to a set of tasks by switching between different predefined behavioral routines. To tackle the challenges of computational complexity associated with such discrete assignment-based approaches, market-based methods~\cite{lin2005combinatorial,parker2007allocation,otte2020auctions,irfan2016auction} have been proposed, where robots can split tasks among them via bidding and auctioning protocols. In scenarios where a large number of robots with limited capabilities are present, decentralized stochastic approaches have been developed where allocations are described in terms of population distributions and are achieved by modifying transition rates among tasks \cite{mather2011macroscopic,berman2009optimized,mayya2019closed}.\par 

As multi-robot tasks get more complex and diverse, robots have been envisioned to take up specialized roles within the team, necessitating the characterization of resource diversity and access within the multi-robot system \cite{abbas2014characterizing,balch2000hierarchic}. In~\cite{prorok2017impact,ravichandar2020strata}, the authors define a community of robot species (robot type), each endowed with specific capabilities, and develop an optimization-based framework to allocate sufficient capabilities to each task. This is realized using transition rates, which the robots use to switch between the different tasks. In comparison, our approach explicitly models how different feature bundles available to the robots might endow them with capabilities required to execute a task. This has the benefit of enduing the allocation framework with a degree of resilience, as will be demonstrated in Section~\ref{subsec:resilience}.\par 
Indeed, adaptivity and resilience are commonly studied aspects of task allocation in multi-robot systems (see, e.g.,~\cite{iocchi2003distributed,lerman2006analysis,palmieri2018self,FatimaS2001Atra}). Typically, adaptivity is incorporated by defining a time-varying propensity of robots to participate in different tasks. These measures of utility are based on predefined objective functions and aim to capture the effectiveness of the robots at performing tasks in real-time~\cite{iijima2017adaptive}. Such frameworks, however, do not account for drastic unexpected failures in the capabilities of the robots, adversarial attacks, or varying environmental conditions that might affect the operations of the robots. Such considerations point to the question of resilience in multi-robot systems, which has been explored in the context of coordinated control tasks~\cite{saulnier2017resilient}, as well as resource-availability in heterogeneous systems~\cite{ramachandran2019resilience}.

Building up on our previous work, presented in \cite{notomista2019constraint,notomista2019optimal,emam2020adaptive}, in this paper we present three distinct novel developments towards the achievement of a resilient task allocation and execution framework. These are: (i) the explicit feature- and capability-based models of robot heterogeneity, which allows for greater flexibility in allocating tasks; (ii) an optimization-based task execution framework which allows robots to execute prioritized tasks accounting for the different features and capabilities they possess; and (iii) a minimum-energy task allocation framework---geared towards long-duration autonomy applications---which leverages the real-time performance of robots in executing the tasks to effectively realize a resilient task allocation framework.

\begin{table*}[t]
 \caption{Notation}
 \label{tab:notation}
 \centering
 \begin{tabular}{|C{4cm}|C{6cm}|C{2cm}|}
 	\hline
    \multicolumn{1}{|c|}{Symbol} & \multicolumn{1}{c|}{Description} & \multicolumn{1}{c|}{Section}\\
    \hline
	$n_r$ & Number of robots & \ref{subsec:problem_formulation}\\
	$n_t$ & Number of tasks & \ref{subsec:problem_formulation}\\
	$n_c$ & Number of capabilities of all robots & \ref{subsec:problem_formulation}\\
	$n_f$ & Number of features of all robots & \ref{subsec:problem_formulation}\\
	$T \in \{0,1\}^{n_t \times n_c}$ & Capability-to-Task mapping& \ref{subsec:tasktocapability}\\
	$A \in \{0,1\}^{n_f \times n_r}$ & Robot-to-Feature mapping& \ref{subsec:robot-to-feature}\\
	$H_k \in [0,1]^{n_{c_k} \times n_f}$ & Feature-to-Capability mapping & \ref{cap2feat}\\
	$F \in \mathbb{R}^{n_c \times n_r}$ & Robot-to-Capability mapping & \ref{subsec:cap2robot}\\
	$S_i \in\R^{n_t\times n_t}$ & Specialization matrix of robot $i$ & \ref{subsec:specializationMatrix}\\
	$\alpha\in\{0,1\}^{n_t\times n_r}$ & Matrix of task priorities & \ref{subsec:task_prioritization}\\
	$\delta_i\in\R^{n_t}$ & Task relaxation for robot $i$ & \ref{subsec:constraint-task-execution}\\
	$\delta\in\R^{n_t n_r}$ & Vector of task relaxation parameters & \ref{subsec:task_prioritization}\\
	$x_i\in\R^{n_x}$ & State of robot $i$ & \ref{sec:te}\\
	$u_i\in\R^{n_u}$ & Control input of robot $i$ & \ref{sec:te}\\
	$x\in\R^{n_x n_r}$ & Ensemble state of $n_r$ robots & \ref{subsec:constraint-task-execution}\\
	$u\in\R^{n_u n_r}$ & Ensemble control input of $n_r$ robots & \ref{sec:te}\\
    \hline
 \end{tabular}
\end{table*}

\section{Encoding Robot Heterogeneity}
\label{sec:encRobotHeter}
The objective of this section is to develop a framework which generates a feasible mapping between robots and their assigned tasks, while explicitly accounting for the heterogeneity in the robots and the different capability requirements of the tasks. We define a novel notion of feasibility based on a newly added feature layer in the description of the robots. Intuitively, a feasible assignment needs to take into account the capabilities needed for the tasks along with the features possessed by the robots. For example, assigning a ground vehicle to an aerial-surveillance task would be considered an unfeasible assignment. Shown in Fig.~\ref{fig:quad} is an example of the three mappings to be introduced in the next subsections. Starting from the left, we begin by introducing the Capability-to-Task mapping ($T$) which contains the task specifications. In turn, each of those capabilities requires any one of various feature-bundles to be exhibited. This is captured in the Feature-to-Capability mapping ($H_k$) through the use of hypergraphs. In subsection~\ref{subsec:robot-to-feature}, we define the Robot-to-Feature mapping ($A$) which maps each robot to the set of features it possesses. Finally, we introduce a way to obtain the Robot-to-Capability mapping ($F$) which is directly used in the task allocation and execution framework. In Table~\ref{tab:notation}, we summarized this notation, together with the one used throughout the paper.

In the context of the presented model of robot heterogeneity based on capabilities and features, we will considered tasks as uniquely defined by a single set of capabilities it requires to be executed. Moreover, we will assume that capabilities (such as flying) are determined by features that the robot possess (such as fixed wings or propellers). The following example is aimed at stressing the difference between tasks, capabilities, and features, which will be used throughout this paper.

\begin{example}[Tasks, capabilities, features]
	Consider aerial and amphibious vehicle robots deployed in an environment where there is a river. In this scenario, crossing the river is not a task as it is not determined by a single set of capabilities. On the other hand, examples of tasks are flying over the river (Task 1) and swimming from one side of the river to the opposite one (Task 2). Aerial robots in the form of fixed-wings aircrafts and quadrotors possess the features (fixed wings and propellers) which endow them with the capabilities to perform Task 1; amphibious vehicles, instead, are able to execute Task 2 thanks to the features (waterproof body and water propellers) which determine the capability of swimming.
\end{example}

\subsection{Capability-to-Task Mapping}
\label{subsec:tasktocapability}

In many applications such as search and rescue and flexible manufacturing, it is necessary for a heterogeneous team of robots to simultaneously accomplish various tasks each requiring a different set of capabilities. For example, a task requiring the delivery of a product from one point to another may require two capabilities: packaging and transportation. We therefore define the mapping from the set of tasks at hand to their respective capabilities as 
\begin{equation} \label{task2cap}
T \in \{0,1\}^{n_t \times n_c},
\end{equation}
where $T_{tk} = 1$ if and only if task $t$ requires capability $k$, and $n_t$ and $n_c$ denote the numbers of tasks and capabilities respectively. Note that the values of $T$ need not necessarily be binary: in Section~\ref{subsec:weightsExt} we present an extension including weights, i.e. with $T_{tk} \in \R_{\ge0}$. In Fig.~\ref{fig:quad} we present an example setup of an assignment problem consisting of two tasks and three capabilities denoted by $t_t$ and $c_k$, respectively. The mapping from tasks to capabilities is a graphical representation of the matrix $T$ in the form of a bipartite graph: a graph whose nodes are split into two disjoint sets and whose edges contain a single node from each of those two sets. On the left-hand side of Fig.~\ref{fig:quad}, those two sets are the tasks and the capabilities, and the information contained in the graph yields the following mapping $T$:
\begin{equation}
T =
\begin{bmatrix}
     1 & 1 & 0 \\
     0 & 0 & 1 \\
\end{bmatrix} 
\end{equation}
Therefore, by looking at the edges incident to $t_1$ in Fig.~\ref{fig:quad} or at the first row of $T$, one can deduce that capabilities $c_1$ and $c_2$ are required by task $t_1$. 

\begin{figure}[t]
\centering
 \includegraphics[width=0.48\textwidth]{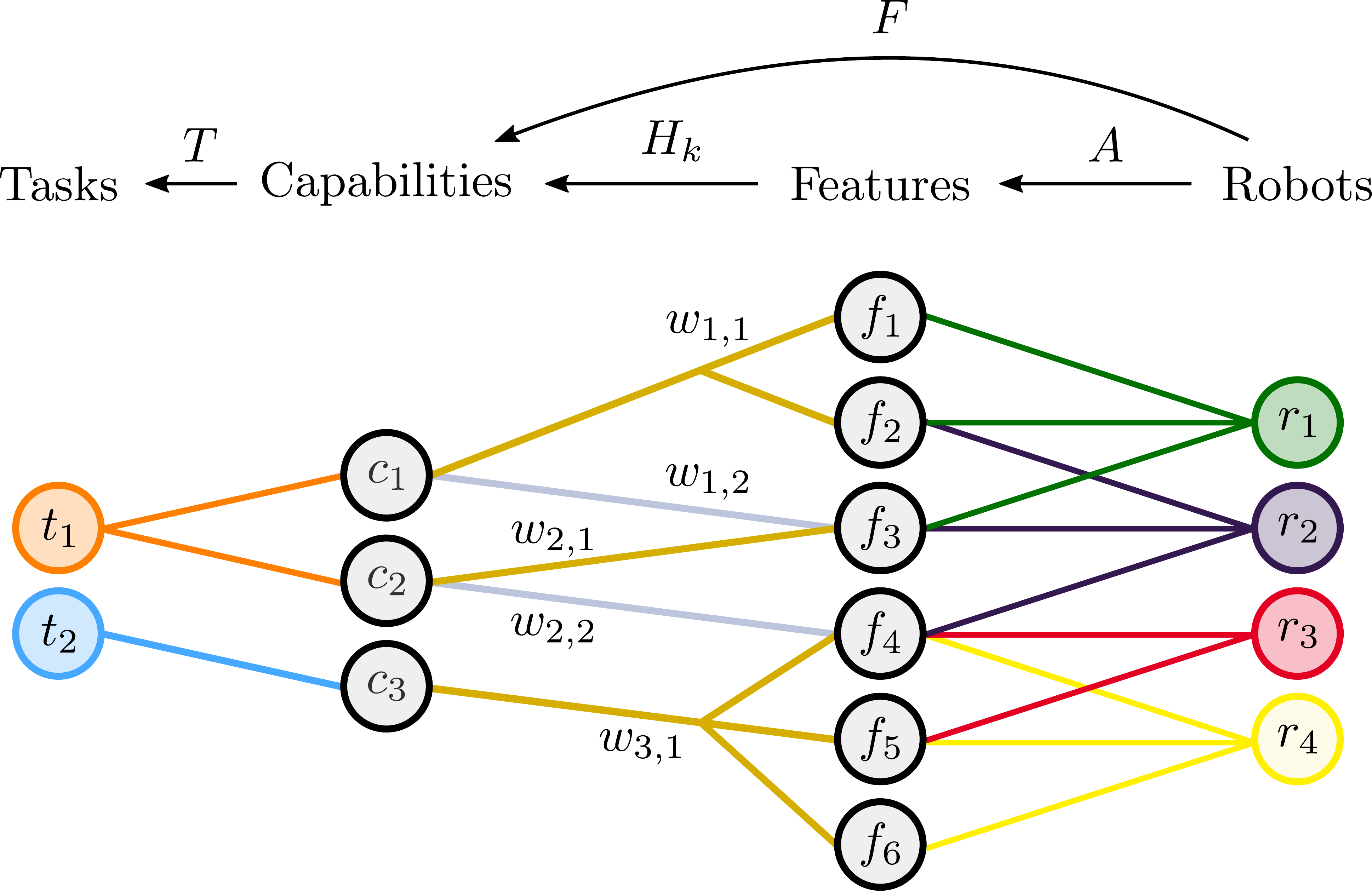}
 \caption{Example of scenario including 2 tasks, 3 capabilities, 6 features and 4 robots shown from left to right. The capabilities to features mapping is shown through the gold and silver hyperedges. Note that not all of the hyperedges need to have the same cardinality.}
 \label{fig:quad}
\end{figure}

\subsection{Robot-to-Feature Mapping} \label{subsec:robot-to-feature}
As mentioned in Section~\ref{sec:intro}, each robot available for assignment possesses a variety of features. For example, an e-puck's features include an IMU and a CMOS camera \cite{epuck}. Therefore, we define the following binary mapping from robots to their respective features:
\begin{equation} \label{feat2agent}
A \in \{0,1\}^{n_f \times n_r},
\end{equation}
where $A_{ij} = 1$ if and only if robot $j$ possesses feature $i$, and $n_r$ and $n_f$ denote the number of robots and features, respectively. Continuing with the example from Fig.~\ref{fig:quad}, the right-most bipartite graph yields the following Robot-to-Feature mapping:
\begin{equation} \label{eq:a_matrix}
A =
\begin{bmatrix}
     1 & 0 & 0 & 0 \\
     1 & 1 & 0 & 0 \\
     1 & 1 & 0 & 0 \\
     0 & 1 & 1 & 1 \\
     0 & 0 & 1 & 1 \\
     0 & 0 & 0 & 1 \\
\end{bmatrix}. 
\end{equation}
By looking at the first column of matrix $A$ above, one can deduce that robot $r_1$ possesses features $f_1$, $f_2$ and $f_3$. Now that we have defined both the Capability-to-Task and Robot-to-Feature mappings, we are ready to introduce the Feature-to-Capability mapping in the following subsection. 

\subsection{Feature-to-Capability Mapping} \label{subsec:cap2feat}
When considering heterogeneous multi-robot systems, it is important to note that two non-identical robots may be able to support the same capability.  In other words, certain robots possessing different sensors and actuators can be interchangeable when it comes to supporting a specific capability. This gives the rise to the need of associating multiple bundles of features to the same capability in a distinguishable manner. To meet this need, we use the notion of a bipartite hypergraph to define the Feature-to-Capability mapping. 

A hypergraph is a graph whose edges are not restricted to a cardinality of two. Hence, we can use one hyperedge to associate a capability with one of the feature bundles that can support it. The mapping between capabilities and features in the middle of Fig.~\ref{fig:quad} is an example of a bipartite hypergraph. The top edge (colored golden) mapping $c_1$ to $f_1$ and $f_2$ indicates that, together, these two features can support capability $c_1$. Note that feature bundles vary in size and consequently, so does the cardinality of the different hyperedges. Therefore, one cannot define a single matrix carrying the information of the entire Feature-to-Capability mapping. This is due to the fact that there is no consistency in terms of edge cardinality nor in the number of edges incident to each capability. Each capability requires its own mapping from its respective hyperedges to the feature space.

Therefore, we define the following row-stochastic matrix, i.e. a matrix where each row sums to $1$:
\begin{equation} \label{cap2feat}
H_k \in [0,1]^{n_{c_k} \times n_f},
\end{equation}
where $k$ denotes the capability index, $H_{k,ij} \neq 0$ if and only if feature $j$ belongs to the feature bundle denoted by hyperedge $i$. $n_{c_k}$ and $n_f$ denote the number of hyperedges incident to capability $k$ and the number of features, respectively.
Revisiting the example setup in Fig.~\ref{fig:quad}, the mappings from capabilities $c_1$, $c_2$ and $c_3$ to the feature space respectively yield:
\begin{equation}
H_1 =
\begin{bmatrix}
     1/2 & 1/2 & 0 & 0 & 0 & 0 \\
     0 & 0 & 1 & 0 & 0 & 0 \\
\end{bmatrix},
\end{equation}
\begin{equation}
H_2 =
\begin{bmatrix}
    0 & 0 & 1 & 0 & 0 & 0 \\ 
    0 & 0 & 0 & 1 & 0 & 0
\end{bmatrix},
\end{equation}
\begin{equation}
H_3 =
\begin{bmatrix}
     0 & 0 & 0 & 1/3 & 1/3 & 1/3 
\end{bmatrix}.
\end{equation}
As explained in the following subsection, normalizing the rows of $H_k$ allows us to verify if the requirements for each capability are met, regardless of the varying feature-capability edge cardinalities. In the next subsection, we utilize the above developed framework to create a mapping from robots to capabilities, which enables task allocation in Section~\ref{sec:te}.

\subsection{Mapping Robots to Capabilities Directly} \label{subsec:cap2robot}

The MRTA algorithm presented in this paper can be referred to as ST-MR-IA (Single-Task robots, Multi-Robot tasks, Instantaneous Assignment) \cite{taxonomy}: in fact, (i) through prioritization, each robot is assigned to a single task, (ii) the tasks can be executed by multiple robots, in a coordinated or independent fashion, and (iii) the allocation of tasks to robots is carried out at each time instant, without planning for future allocations. As discussed in Section~\ref{subsec:lit_review}, previous approaches to solving ST-MR-IA MRTA problems assume knowledge of the direct mapping from capabilities to robots encoded by a matrix
\begin{equation}
    F \in \mathbb{R}^{n_c \times n_r}
\end{equation}
where $F_{kj} \neq 0$ if and only if robot $j$ can support capability $k$. Therefore, in this subsection, we state the required condition under which robot $j$ can indeed support capability $k$, and derive the matrix $F$ required by such algorithms based on this condition. Notice that the framework only accounts for a finite number of capabilities and features relevant to the required tasks, so the computation of $F$ remains tractable.

As mentioned above, a capability $k$ can be supported by a number of feature bundles. Consequently, a robot must possess all the features in at least one of the bundles associated with a capability in order to support it. Hence, we say robot $i$ supports capability $k$ if and only if it possesses all the features within a hyperedge associated with capability $k$. For example, in Fig.~\ref{fig:quad}, robot $r_1$ can support capability $c_1$ since it possesses features $f_1$ and $f_2$ included in the top hyperedge. On the other hand, robot $r_3$ cannot support capability $c_3$ since it only possesses features $f_4$ and $f_5$ but not $f_6$. We define the feasibility vector $F_k$ capturing which robots can satisfy capability $k$ as:
\begin{equation} \label{eq:cap2agent}
F_k = \max(\mathrm{kron}_1({H_{k}A})),
\end{equation}
where $\mathrm{kron}_n$ denotes the shifted Kronecker delta function
\begin{equation}
\label{eq:kron}
\mathrm{kron}_n(x) = \begin{cases} 
      1 & \text{if } x = n \\
      0 & \text{otherwise}
   \end{cases}
\end{equation}
applied element-wise. The function $\mathrm{kron}_1$ is introduced to eliminate cases where robots have an incomplete portion of the features in a hyperedge. Moreover, the $\max$ operator is intended column-wise, and serves to check whether a robot possesses all the features from at least one bundle. Note that using the $\max$ operator in the case of a robot satisfying a capability through multiple hyperedges selects only one of those edges, which will become relevant when we introduce weights in the next section.

Shifting our attention to the example from Fig.~\ref{fig:quad},  we can compute the feasibility vector $F_3$ corresponding to $c_3$:
\begin{equation}
H_{3}A =
\begin{bmatrix}
     0 & 1/3 & 2/3 & 1
\end{bmatrix}
\end{equation}
whose ${ij}$-th component is the proportion of features that robot $j$ possesses belonging to hyperedge $i$ incident to capability $k$. Therefore, in this case, robot $r_3$ possesses only $2/3$ of the features in the only bundle associated with $c_3$, and therefore cannot support that capability. We thus obtain the following feasibility vector for $c_3$:
\begin{equation}
F_3 = \max(f(H_{3}A)) =
\begin{bmatrix}
     0 & 0 & 0 & 1  \\
\end{bmatrix}.
\end{equation}
As illustrated above, if $F_{k,j} = 1$, then robot $j$ can support capability $k$. Therefore, by concatenating all the vectors $F_k$, we obtain the desired linear mapping from capabilities to robots:
\begin{equation}
F = \begin{bmatrix}
F_{1}\tr & F_{2}\tr & \ldots & F_{n_c}\tr
\end{bmatrix}\tr.
\end{equation}

As such, we can define a feasible assignment as one where all the capabilities required by each task $t$ are at least supported by a given number robots assigned to task $t$, i.e.
\begin{equation}
\label{eq:FT}
\sum_{i \in \mc R_t} F_{-,i} \geq T_{t,-},
\end{equation}
where the notation $T_{t,-}$ and $F_{-,i}$ is used to denote the $t$-th row of $T$ and the $i$-th column of $F$. The inequality in \eqref{eq:FT} holds element-wise, and $\mc R_t$ denotes the set of robots assigned to task $t$. $T_{t,k} = n$ indicates that at least $n$ of the robots assigned to task $t$ need to exhibit capability $k$. Notice that, in order to ensure the satisfaction of inequality~\eqref{eq:FT} for some set of robots $\mc R_t$, it is necessary that there are enough features available among the robots as are required for the execution of each task.
 
\subsection{Weights Extension} \label{subsec:weightsExt}

In Subsection~\ref{subsec:cap2feat}, a binary mapping from features to capabilities leveraging the notion of hyper-edges was presented. In other words, depending on its features, a robot either fully exhibited a capability or did not at all. However, in many real-world applications this mapping is not necessarily binary. For example, continuous tracks perform better than ordinary wheels in navigating uneven terrains and therefore the hyper-edge containing the continuous track feature should be assigned a higher weight.
This notion can be captured by introducing a weight associated to each hyper-edge, leading to the following more general form of \eqref{eq:cap2agent}:
\begin{equation}
\label{eq:featuresmapping}
F_k = \max(W_{k}\mathrm{kron}_1(H_{_k}A)),
\end{equation}
where $W_{k}$ is diagonal matrix whose diagonal elements $w_{k,1},\ldots,w_{k,n_{c_k}}$ specify the quality of each hyper-edge at exhibiting capability $k$, as depicted in Fig.~\ref{fig:quad}. Given this refined definition of $F_k$, the inequality in \eqref{eq:FT} ensures that each capability required by task $t$ is exhibited by at least one robot assigned to task $t$. In the next subsection, we introduce a \emph{specialization matrix} which encodes information about which tasks a given robot is a valid candidate for assignment.   

\subsection{Specialization Matrix} \label{subsec:specializationMatrix}
To conclude the model of robot heterogeneity used within the task allocation framework proposed in this paper, we now define the requirements for a robot to be considered as a potential candidate for a task. As opposed to our previous work \cite{notomista2019optimal}, where the specialization matrices were assumed to be given, we leverage the above developed feature and capability models to compute the specialization matrix of robot $i$ as follows:
\begin{equation}
\label{eq:specialization}
S_i = \mathrm{diag}(\mathbbm{1}_{n_t}-\mathrm{kron}_0(T F_{-,i}))\in\R^{n_t\times n_t},
\end{equation}
where, for $m\in\R^n$, $\mathrm{diag}(m) = M\in\R^{n\times n}$ such that
\begin{equation}
M_{ij} = \begin{cases}
m_i & \text{if }i=j\\
0 & \text{otherwise},
\end{cases}
\end{equation}
$\mathbbm{1}_{n_t}$ is a vector of dimension $n_t$ whose entries are all equal to 1, and $\mathrm{kron}_0$ denotes the Kronecker delta function defined in \eqref{eq:kron} applied element-wise. As a result, the specialization of robot $i$ towards task $j$, $s_{ij}$, is given by
\begin{equation}
s_{ij} = \begin{cases} 
      1 & \text{if } T_{j,-} F_{-,i}  > 0 \\
      0 & \text{otherwise},
   \end{cases}
\end{equation}
i.e. $s_{ij} = 1$ if robot $i$ exhibits at least one capability required by task $j$. The motivation behind this choice is two-fold: robots are allowed to combine their capabilities to satisfy a task, and there is no notion of priority between capabilities (i.e. exhibiting capability 1 is more or less crucial than exhibiting capability 2 and 3). The former indicates that if a robot exhibits even a single capability relevant to the task, it may still be able to contribute, whereas the latter indicates that there is no possible ordering of the candidates in terms of specialization.

Finally, as will be shown in Section~\ref{sec:exp}, the specialization matrix can be adapted on-the-fly. For example, in the case a robot loses a feature (e.g. its camera is malfunctioning), by removing the edges between the robot and the feature, we can re-compute which capabilities the malfunctioning robot can still exhibit.

\section{Task Execution and Prioritization}
\label{sec:te}

This section develops a minimum energy task allocation framework, through prioritization and execution, that explicitly accounts for the heterogeneity of the robots expressed in terms of their capabilities, as well as specifications on the capabilities required to execute each task. Moreover, we demonstrate how the proposed task allocation framework introduces a degree of resilience, allowing the robots to react, for instance, to component failures and, more generally, to unmodeled or unexpected environmental conditions.

As stated in Section~\ref{subsec:problem_formulation}, we consider a team of $n_r$ robots tasked with executing $n_t$ different tasks in the environment. We model the dynamics of each robot $i \in \{1,\ldots,n_r\}$ with a control-affine dynamical system:
\begin{equation}
\label{eq:controlaffine}
    \dot x_i = f(x_i) + g(x_i) u_i
\end{equation}
where $f$ and $g$ are locally Lipschitz continuous vector fields,  $x_i \in \mc X \subseteq \mathbb{R}^{n_x}$ is the state of the robot, and $u_i \in \mc U \subseteq \mathbb{R}^{n_u}$ is the input. Note that, in this paper, we assume that all robots obey the same dynamics given in~\eqref{eq:controlaffine}, however, the entire formulation can be extended in a straightforward fashion to the case where individual robots have different dynamics. As done in \cite{notomista2020set}, we use Control Barrier Functions (CBFs) (see~\cite{ames2019control} for a review on the subject) to encode the set-based tasks that the robots are required to execute. To this end, in the following we briefly recall the definition and the main properties of CBFs, which will be used in the rest of the paper to formulate the task prioritization and execution framework.

\begin{definition}[\cite{ames2019control}]
\label{def:cbf}
Let $\mc C \subset \mc D \subset \R^n$ be the zero superlevel set of a continuously differentiable function $h\colon\mc D \to \R$. Then $h$ is a control barrier function (CBF) if there exists an extended class $\mc K_\infty$ function $\gamma$\footnote{An extended class $\mc K_\infty$ function is a continuous function $\gamma : \R \to \R$ that is strictly increasing and with $\gamma(0) = 0$.} such that, for the control affine system $\dot x = f(x) + g(x) u$, $x\in\R^{n_x}$, $u\in\R^{n_u}$, one has
\begin{align}
\label{eqn:cbf:definition}
\sup_{u \in \mc U}  \left\{ L_f h(x) + L_g h(x) u \right\} \geq - \gamma(h(x)).
\end{align}
for all $ x \in \mc D$.
\end{definition}
The notation $L_f h(x)$ and $L_g h(x)$ are used to denote the Lie derivative of $h$ along the vector fields $f$ and $g$, respectively. Given this definition of CBFs, the following theorem highlights how they can be used to ensure both set forward invariance and stability~\cite{xu2015robustness}.

\begin{theorem}[\cite{ames2019control}]
\label{thm:cbf}
Let $\mathcal{C} \subset \R^n$ be a set defined as the zero superlevel set of a continuously differentiable function $h: \mc D \subset \R^n \to \R$.  If $h$ is a control barrier function on $\mc D$ and $\frac{\partial h}{\partial x}(x)\neq 0$ for all $x\in\partial \mc C$, then any Lipschitz continuous controller $u(x) \in \{ u \in \mc U \colon L_f h(x) + L_g h(x) u + \gamma(h(x)) \geq 0\}$ for the system $\dot x = f(x) + g(x) u$, $x\in\R^{n_x}$, $u\in\R^{n_u}$, renders the set $\mc C$ forward invariant.  Additionally, the set $\mc C$ is asymptotically stable in $\mc D$.
\end{theorem}

The results of this theorem will be used in the remainder of this section to design a control framework that allows a heterogeneous multi-robot system to prioritize and perform a set of tasks that need to be executed.

\subsection{Constraint-Driven Task Execution}
\label{subsec:constraint-task-execution}
The formulation adopted in this paper in terms of extended set-based tasks \cite{notomista2020set} allows us to encode a large variety of tasks: these are tasks characterized by a set, which is to be rendered either forward invariant (usually referred to as \emph{safety} in dynamical system theory \cite{ames2019control}), or asymptotically stable, or both. The results recalled above suggest the use of CBFs to encode these kinds of tasks. Indeed, CBFs have been successfully used to encode a variety of such tasks for different robotic platforms, ranging from coordinated control of multi-robot systems \cite{notomista2019optimal} to multi-task prioritization for robotic manipulators \cite{notomista2020set}. In particular, in \cite{notomista2020set} the definition of extended set-based tasks, i.e. tasks which consist in the state $x$ approaching a set (stability) or remaining within a set (safety), is formalized.

As shown in \cite{notomista2019constraint}, among the extended set-based tasks, there is a class of coordinated multi-robot tasks which are executed through the minimization of a cost function, realized, for instance, by gradient-flow-like control laws \cite{cortes2017coordinated}. These types of tasks can be recognized to be extended set-based tasks where the set of stationary points of the cost function has to be rendered asymptotically stable. In \cite{notomista2019constraint}, it is shown how the execution of these tasks can be turned into a constrained optimization problem---a formulation amenable for long-duration robot autonomy \cite{egerstedt2018robot}.

To make matters more concrete, consider the continuously differentiable positive definite (energy-like) cost $J: \mathbb{R}^{n_x} \rightarrow \mathbb{R}$, which is a function of the robot state $x_i$, whose dynamics are assumed to be control affine, as in \eqref{eq:controlaffine}. Then, it is shown in \cite{notomista2019constraint} how the execution of the task characterized by the minimization of the cost function $J$ can be realized by solving the following constrained optimization problem:
\begin{align}
\label{eq:minJ}
    \minimize_{u_i,\delta_i} & \|u_i\|^2 + \delta_i^2 \\
    \subjto & L_fh(x_i) + L_gh(x_i)u_i \geq -\gamma(h(x_i))-\delta_i,
\end{align}
where the task is encoded by the constraint in which $h(x_i) = -J(x_i)$ is a CBF that renders the safe set
\begin{align}
\mathcal C &= \{ x_i\in\R^{n_x} \colon h(x_i)\ge0 \}\\ &= \{ x_i\in\R^{n_x} \colon J(x_i)\le0 \}\\
&= \{ x_i\in\R^{n_x} \colon J(x_i)=0 \}
\end{align}
asymptotically stable. In \eqref{eq:minJ}, $\gamma$ is an extended class $\mathcal K_\infty$ function and $\delta_i$ is a slack variable which quantifies the extent to which the constraint can be relaxed. In cases where the completion of a task (a stationary point of $J$) is characterized by a strictly positive value of the cost $J$, $\delta_i$ ensures that the optimization program \eqref{eq:minJ} remains feasible (see \cite{notomista2019constraint}).

In multi-task multi-robot settings, this framework naturally allows robots to combine multiple constraints, each representing a task, into a single framework. For tasks encoded via CBFs $h_m, m \in \{1,\ldots,n_t\}$, the constraint-based optimization problem for robot $i$ can be written as,
\begin{align}
    \label{eq:minJ_m}
    \minimize_{u_i,\delta_i} & \|u_i\|^2 + \|\delta_i\|^2 \\
    \subjto & L_fh_m(x) + L_gh_m(x)u \geq -\gamma(h_m(x))-\delta_{im}\\
&\hspace{4.4cm}\forall m \in \{1\ldots n_t\}, 
\end{align}
where $\delta_i = [\delta_{i1},\ldots,\delta_{in_t}]^T$ represents the slack variables corresponding to each task being executed by robot~$i$. The tasks encoded by the CBFs $h_m(x)$ are not restricted to be dependent only on the state of robot $i$, but rather on the ensemble state of the robots $x=[x_1\tr,\ldots,x_{n_r}\tr]\tr$, thus allowing the framework to encompass coordinated multi-robot tasks.

With this framework in place, the slack variables $\delta_i$ present a natural way of encoding task priorities for the individual robots. This will be the subject of the next section, where the main task allocation framework is presented.

\subsection{Task Prioritization and Execution Algorithm}
\label{subsec:task_prioritization}
In section \ref{sec:encRobotHeter}, we presented a framework to model robot heterogeneity---exhibited in the different suitability that each robot has for different tasks---starting from the lower level concepts of robot features and capabilities. In this section, we leverage the expressiveness of this model in order to render the task prioritization framework, presented in \cite{notomista2019optimal} and improved in \cite{emam2020adaptive}, resilient.

The optimization-based formulation extends the one in \eqref{eq:minJ_m} as follows:
\begin{subequations}\label{eq:allocationalgorithm}
\begin{flalign}
\text{\bf Task allocation optimization problem (MIQP)} \tag{\ref{eq:allocationalgorithm}}&&\label{eq:allocationalgorithmactual}
\end{flalign}
\vspace{-0.75cm}
\begin{align}
\minimize_{u,\delta,\alpha} & \sum_{i = 1}^{n_r} \left( C \| \Pi_i \alpha_{-,i}\|^2 +  \|u_i\|^2 + l \|\delta_i \|_{S_i}^2 \right) \label{eq:miqp:a}\\
\subjto & L_f h_{m}(x) + L_g h_{m}(x) u_i\\
& \qquad \geq -\gamma(h_{m}(x)) - \delta_{im} \label{eq:miqp:b}\\
&\Theta\delta_i + \Phi\alpha_{-,i}  \le \Psi \label{eq:miqp:c}\\
&\mathbbm{1}_{n_t}\tr\alpha_{-,i} \le 1 \label{eq:miqp:d}\\ 
& F \alpha_{m,-}\tr \geq T_{m,-}\tr \label{eq:miqp:e}\\
& n_{r,m,\text{min}} \le \mathbbm{1}\tr \alpha_{m,-}\tr \leq n_{r,m,\text{max}} \label{eq:miqp:f}\\
&\|\delta_i\|_\infty \leq \delta_\text{max} \label{eq:miqp:g}\\
&\alpha \in \{0,1\}^{n_t\times n_r} \label{eq:miqp:h}\\
&\hspace{2cm}\forall i \in \{1\ldots n_r\},~\forall m \in \{1\ldots n_t\}, 
\end{align}
\noeqref{eq:miqp:a}\noeqref{eq:miqp:b}\noeqref{eq:miqp:c}\noeqref{eq:miqp:d}\noeqref{eq:miqp:e}\noeqref{eq:miqp:f}\noeqref{eq:miqp:g}\noeqref{eq:miqp:h}
\end{subequations}
\hspace{-0.14cm}where $C,l\in\R_{\ge0}$ are parameters of the optimization, $\delta_\text{max}$ signifies the maximum extent to which each task constraint can be relaxed, and $\gamma$ is a continuously differentiable class $\mathcal{K}_\infty$ function. The matrix $\Pi_i$ is a projection matrix defined in \eqref{eq:pi} to account for the heterogeneous capabilities of the multi-robot system, as explained in detail later.

First of all, as done in Section~\ref{sec:encRobotHeter}, the symbols $X_{i,-}$ and $X_{-,j}$ denote the $i$-th row and the $j$-th column of the matrix $X$, respectively. The introduction of the matrix of task priorities $\alpha\in\{0,1\}^{n_t\times n_r}$ in the optimization problem is what determines the prioritization (and, therefore, the allocation) of the tasks for each robot. This is realized through the constraint \eqref{eq:miqp:c}, where the matrices $\Theta\in\R^{\frac{n_t(n_t-1)}{2}\times n_t}$ and $\Phi\in\R^{\frac{n_t(n_t-1)}{2}\times n_t}$, and the vector $\Psi\in\R^{\frac{n_t(n_t-1)}{2}}$, enforce constraints among different components of the vectors of task relaxation parameters $\delta_i$. As extensively discussed in \cite{notomista2019optimal}, the constraint
\begin{equation}
\label{eq:deltaalphaconstraint}
\delta_{in} \geq \kappa\big (\delta_{im} - \delta_\text{max}(1 - \alpha_{mi}) \big ),~~ n\neq m,
\end{equation}
that can be written as \eqref{eq:miqp:c}, realizes the following two implications:
\begin{equation}
\alpha_{im} = 1 \implies \delta_{im} \leq \frac{1}{\kappa}\delta_{in}\quad \forall n\in \{1,\ldots,n_t\}\setminus\{m\},
\end{equation}
which implies that task $m$ has highest priority for robot $i$, and
\begin{equation}
\alpha_{im} = 0 \implies \delta_{im}\le\delta_\text{max}+\frac{1}{\kappa}\delta_{in}\quad \forall n\in \{1,\ldots,n_t\}\setminus\{m\},
\end{equation}
which implies that task $m$ does not have the highest priority for robot $i$. In fact, in light of constraint \eqref{eq:miqp:g}, no further constraints are enforced on $\delta_{im}$, since $\delta_\text{max}$ is the maximum value $|\delta_{im}|$ is allowed to achieve\footnote{Note that constraint \eqref{eq:miqp:g} might cause \eqref{eq:allocationalgorithm} to become infeasible. However, when the state of the robots evolve in a compact set $\mc X$, as the functions encoding the tasks are continuously differentiable, choosing $\max_{m\in\{1,\ldots,n_t\}} \max_{x\in\mc X}\{h_{m}(x)\}\le\delta_\text{max}<\infty$ guarantees that the task allocation optimization problem \eqref{eq:allocationalgorithm} is always feasible.}. Notice further that, for the way it is used in \eqref{eqn:te_2}, the optimal value of $\delta$ will always be non-negative (see also analyses in \cite{notomista2019constraint, notomista2019optimal}).

The constraint \eqref{eq:miqp:d} is used to ensure that each robot has at most one task to be executed with highest priority, making the task prioritization formulation effectively a task allocation. Notice that, compared to our previous work \cite{notomista2019optimal}, \eqref{eq:miqp:d} is here turned from an equality into an inequality constraint. In \cite{notomista2019optimal}, this constraint was used to ensure that no feasible solution consisted in robots trading off task execution for energy saving. In the presented, enhanced, formulation, this is not necessary anymore thanks to constraint \eqref{eq:miqp:e}---whose meaning will be described in the following. Consequently, we can now account for situations where no tasks are allocated to some of the robots, implementing, as a matter of fact, the concept of \emph{autonomy-on-demand} in the context of task allocation.

The constraint \eqref{eq:miqp:e} is what allows us to specify the minimum capabilities required for each task, expressed by the matrices $T$ and $F$ defined in Sections~\ref{subsec:tasktocapability} and \ref{subsec:cap2robot}, respectively. Moreover, the constraint \eqref{eq:miqp:f} allows us to enforce the minimum and maximum number of robots required for each task, thus giving a lot of flexibility and versatility to be utilized in many different application scenarios. In Section~\ref{sec:exp}, experiments performed on a real multi-robot system will showcase the use of these constraints.

As in our previous works \cite{notomista2019optimal} and \cite{emam2020adaptive}, the cost of the optimization problem \eqref{eq:allocationalgorithmactual} is composed of 3 terms. The last two terms in \eqref{eq:miqp:a} correspond to the control effort spent by the robots and the magnitude of the relaxation parameters, respectively. The former enables our framework to be compatible with long-duration autonomy applications. More specifically, robots can remain operational over sustained periods of time by minimizing the energy spent to perform a task---which is proportional to control effort---together with enforcing energy constraints as, e.g., in \cite{notomista2020persistification,fouad2020energy}. The latter, instead, ensures that the tasks to which the robots have been assigned get indeed executed, thanks to constraint \eqref{eq:miqp:b}. The norm of $\delta_i$ corresponding to robot $i$ is weighted by the specialization matrix of robot $i$, $S_i$. This way, the relaxation variables corresponding to tasks that robot $i$ is not capable of performing (i.e. with a low value of the entry of the specialization matrix) are weighted accordingly less.

Finally, the first term in \eqref{eq:miqp:a} is introduced to penalize \emph{bad allocations} of tasks to robots, as explained in the following. The matrix $\Pi_i$ is defined as follows:
\begin{equation}
\label{eq:pi}
\Pi_i = I_{n_t} - S_i S_i^\dagger,
\end{equation}
where $I_{n_t}$ is the $n_t\times n_t$ identity matrix, and $S_i^\dagger$ is the right Moore-Penrose inverse~\cite{penrose1955generalized} of the specialization matrix $S_i$ of robot $i$. It is easy to see that $\Pi_i$ is the projector onto the orthogonal complement of the subspace of specializations possessed by robot $i$. Assume, for example, that robot $i$ has no specialization at all at performing task $k$ (i.e. $s_{ik}=0$) and has a non-zero specialization $s_{ij}$ of performing task $j$, $j\neq k$. Then, its specialization matrix $S_i$ will be given by:
\begin{equation}
S_i = \mathrm{diag}([s_{i1},\ldots,s_{i\,k-1},0,s_{i\,k+1},\ldots,s_{in_t}]),
\end{equation}
and 
\begin{equation}
\Pi_i = \mathrm{diag}([\underbrace{0,\ldots,0}_{k-1},1,\underbrace{0,\ldots,0}_{n_t-k}]).
\end{equation}
Then, the projector $\Pi_i$ in the cost \eqref{eq:miqp:a} will contribute to a non-zero cost when the components of $\alpha_i$ corresponding to tasks that robot $i$ has no specialization to perform are not zero, i.e. when robot $i$ has been assigned to a task that it is not able to perform---referred above as a \emph{bad allocation}.

\begin{remark}[Centralized Mixed-Integer Quadratic Program]
Notice that in \eqref{eq:allocationalgorithmactual} there is a coupling between the robots through the cost as well as the constraints. This means that the task allocation framework has to be solved in a centralized fashion. Moreover, the matrix of task priorities $\alpha$ is integer. This renders \eqref{eq:allocationalgorithmactual} a mixed-integer quadratic program (MIQP). A QP-relaxation approach, as well as ways of solving this framework in a decentralized way, are discussed in \cite{notomista2019optimal}. In Section~\ref{sec:mixedc-dec} of this paper, we show how the proposed MIQP can be solved in a mixed centralized/decentralized fashion, and we analyze the performances compared to the centralized approach.
\end{remark}

\begin{remark}[Time-varying and sequential tasks]
\label{rmk:timevarying}
Expressing tasks by means of control barrier functions, besides rendering the task execution and allocation particularly amenable for online-optimization-based controllers, allows us to account for time-varying and sequential tasks, comprised by a sequence of sub-tasks, as well. In fact, the time-varying extension of control barrier functions (see, e.g., \cite{notomista2020persistification}) can be leveraged to consider tasks which have an explicit dependence on time. In the experimental section, we show how this extension of the proposed task allocation and execution framework can be used to implement state-trajectory-tracking tasks.

Moreover, thanks to the pointwise-in-time nature of the developed optimization program, tasks can be removed and inserted in a continuous fashion, as demonstrated in \cite{notomista2020set}. This allows for a flexible implementation of sequential tasks which require the completion of a sub-task before another one can be started, as done in \cite{li2018formally}. In the same way, the features and the specialization of the robots towards different tasks can be modified during the execution of the task. In the next subsection, we present an approach to leverage time-varying specialization in order to adapt to disturbances or modeled phenomena in the environment.
\end{remark}

The following algorithm summarizes the application of the optimization-based allocation and execution framework to a multi-robot system with heterogeneous capabilities.

\begin{algorithm}
	\caption{Task allocation and execution}
	\label{alg:1}
 	\begin{algorithmic}[1]
	\Require
		\Statex Tasks $h_m,~m\in\{1,\ldots,n_t\}$
		\Statex Mappings $F$, $T$
		\Statex Parameters $n_{r,m,\text{min}}$, $n_{r,m,\text{max}}$, $\delta_\text{max}$, $C$, $l$
	\State Evaluate $S_i,~\forall i\in\{1,\ldots,n_r\}$ \Comment \eqref{eq:specialization}
	\While{true}
		\State Get robot state $x_i,\forall i\in\{1,\ldots,n_r\}$
		\State Compute robot input $u_i, \forall i\in\{1,\ldots,n_r\}$ \Comment \eqref{eq:allocationalgorithmactual}
		\State Send input $u_i,\forall i\in\{1,\ldots,n_r\}$ to robots and execute
	\EndWhile
 	\end{algorithmic}
\end{algorithm}

We conclude this subsection by showcasing the execution of Algorithm~\ref{alg:1} in an explanatory example featuring the use of the allocation constraints described so far.

\begin{example}
\label{ex:sim1}
Consider 4 mobile robots moving in a 2-dimensional space, tasked with performing 2 tasks. For clarity of exposition, in this example, we modeled each robot $i$ as single integrator $\dot x_i = u_i$---so $f$ and $g$ in \eqref{eq:controlaffine} are the zero and identity map, respectively---and each task consists in going to a point of the state space.
\begin{figure}
    \centering
    \subfloat[]{\label{subfig:ex1:traj1}\includegraphics[width=0.2\textwidth]{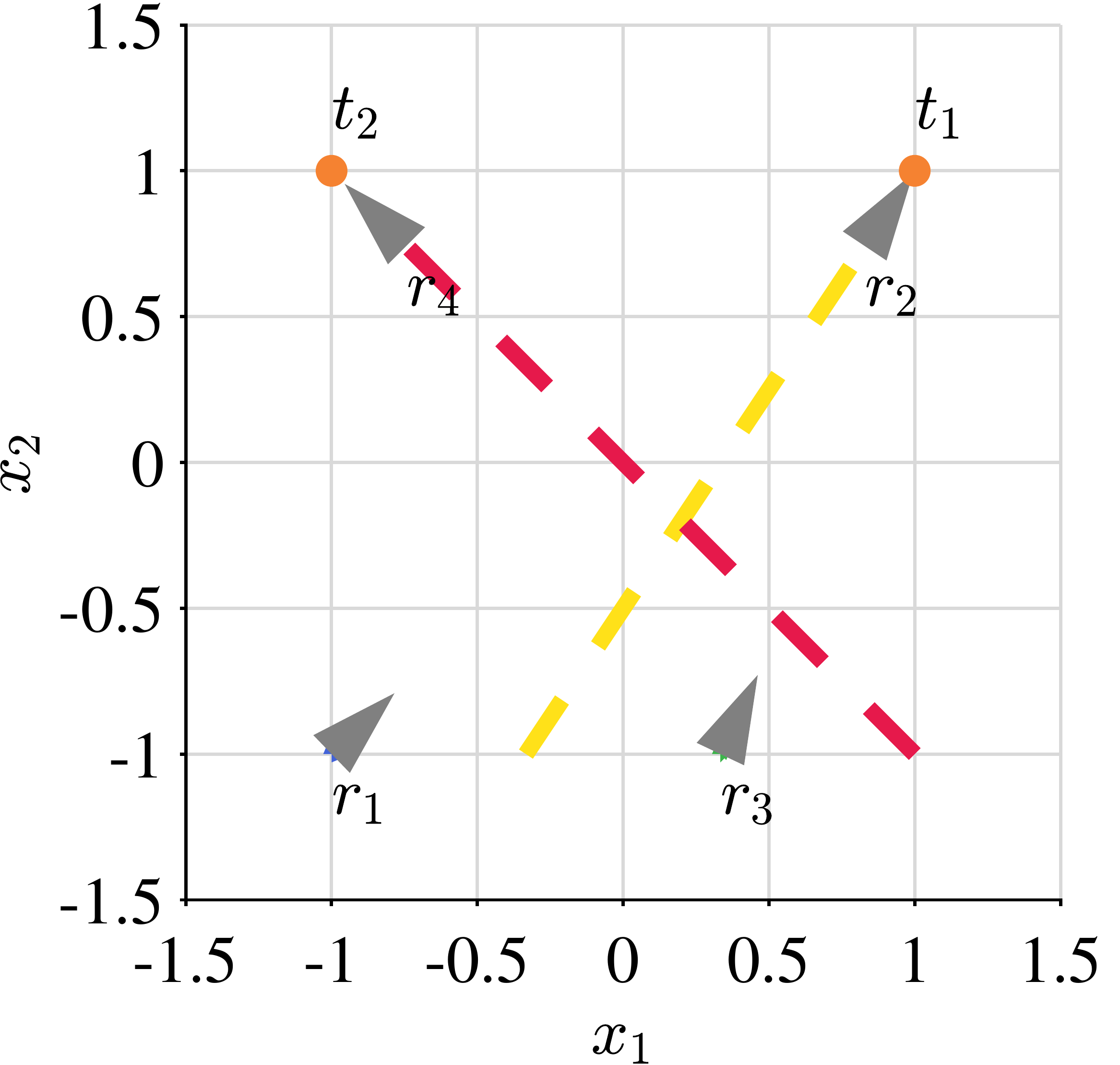}}%
    \subfloat[]{\label{subfig:ex1:traj2}\includegraphics[width=0.2\textwidth]{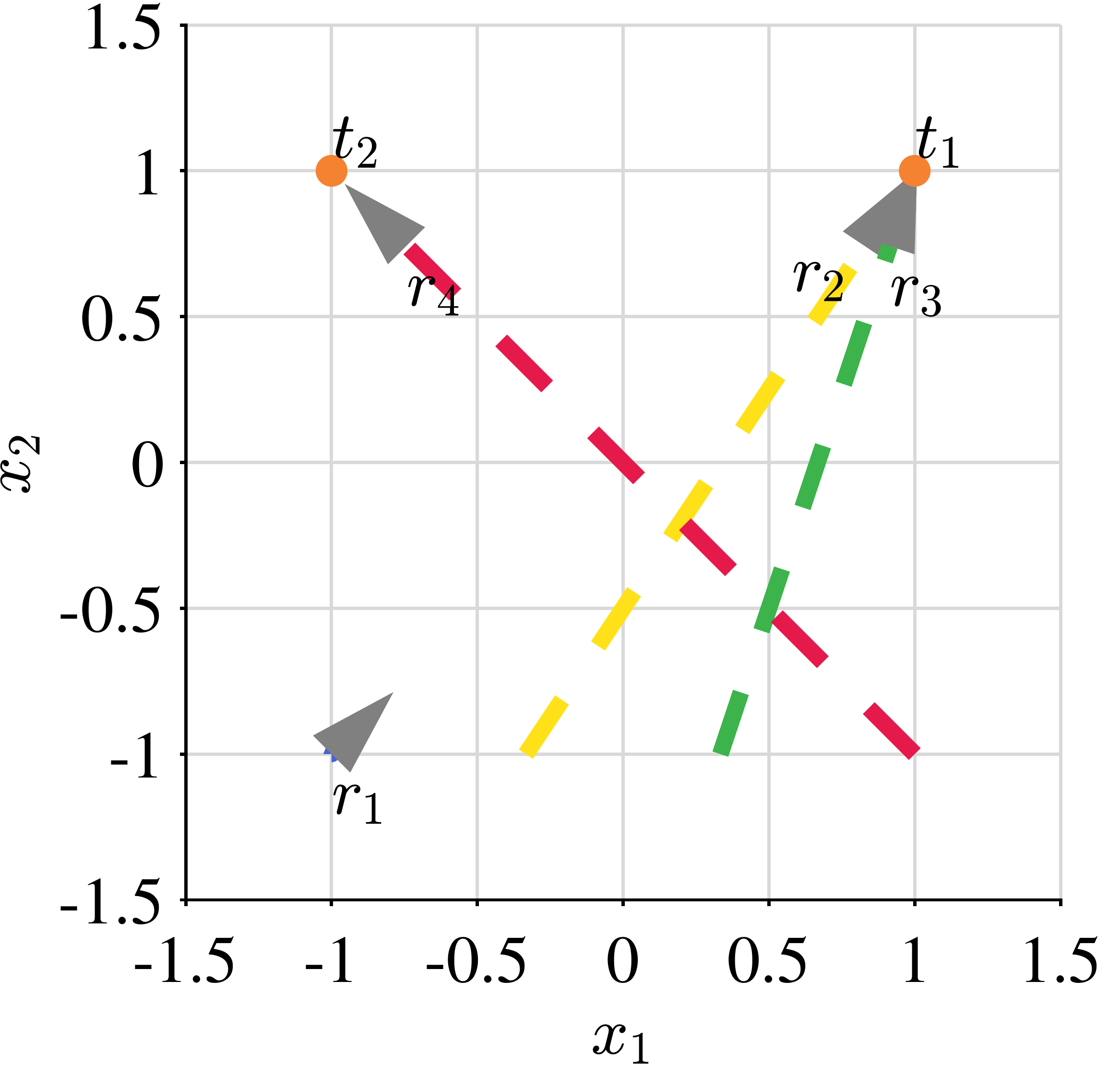}}
    \caption{Task allocation and execution (Example~\ref{ex:sim1}). In Fig.~\protect\ref{subfig:ex1:traj1}, 4 robots (gray triangles) have to be allocated to 2 tasks and, as a result of the execution of \eqref{eq:allocationalgorithmactual}, robots $r_2$ and $r_4$ are assigned to task $t_1$ and $t_2$, respectively, based on their specialization introduced in Fig.~\ref{fig:quad}. In Fig.~\protect\ref{subfig:ex1:traj2}, the additional constraint that at least 2 robots are required to execute task $t_1$ is introduced ($n_{r,1,\text{min}}=2$), and robot $r_3$ is picked together with $r_2$ to perform $t_1$. The resulting trajectories of the robots are depicted as dashed lines.}
    \label{fig:ex1}
\end{figure}
In Fig.~\ref{fig:ex1}, the robots are depicted as gray triangles and labeled $r_1$ to $r_4$, whereas the locations corresponding to the tasks are labeled $t_1$ and $t_2$. The features, capabilities, and task mappings have been set as in Fig.~\ref{fig:quad}, where the numerical quantities are given in Section~\ref{sec:encRobotHeter}. So, per \eqref{eq:specialization}, robots $r_1$, $r_2$ and $r_3$ are only specialized to perform task $t_1$, while robot $r_4$ is specialized to perform both tasks.

For the scenario depicted in Fig.~\ref{subfig:ex1:traj1}, tasks $t_1$ and $t_2$ need to be executed and there is no further constraint on the amount of capability or the number of robots required for a task. As a result of the execution of Algorithm~\ref{alg:1}, the trajectories (red and yellow) show two of the robots performing the two tasks. In particular, robot $r_4$ is assigned to task $t_2$, while robot $r_2$ has been allocated to task $t_1$ (the only one it can perform). Robots $r_1$ and $r_3$ have remained to their initial positions with no task assigned to them, as $r_2$ and $r_4$ were already satisfying the task constraints in \eqref{eq:allocationalgorithmactual}.

In the scenario depicted in Fig.~\ref{subfig:ex1:traj2}, instead, $n_{r,1,\text{min}}=2$, i.e. at least 2 robots are required for the execution of task $t_1$. Driven by the control inputs $u_2$ and $u_3$ calculated according to Algorithm~\ref{alg:1}, robots $r_2$ and $r_3$ are assigned to task $t_1$, while $r_4$ is assigned to $t_2$ (as it is the only robot possessing the specialization for it), as can be seen in Fig.~\ref{subfig:ex1:traj2}.
\end{example}

\subsection{Resilience of the Task Allocation Algorithm}
\label{subsec:resilience}

\begin{figure}
\centering
\includegraphics[width=0.28\textwidth]{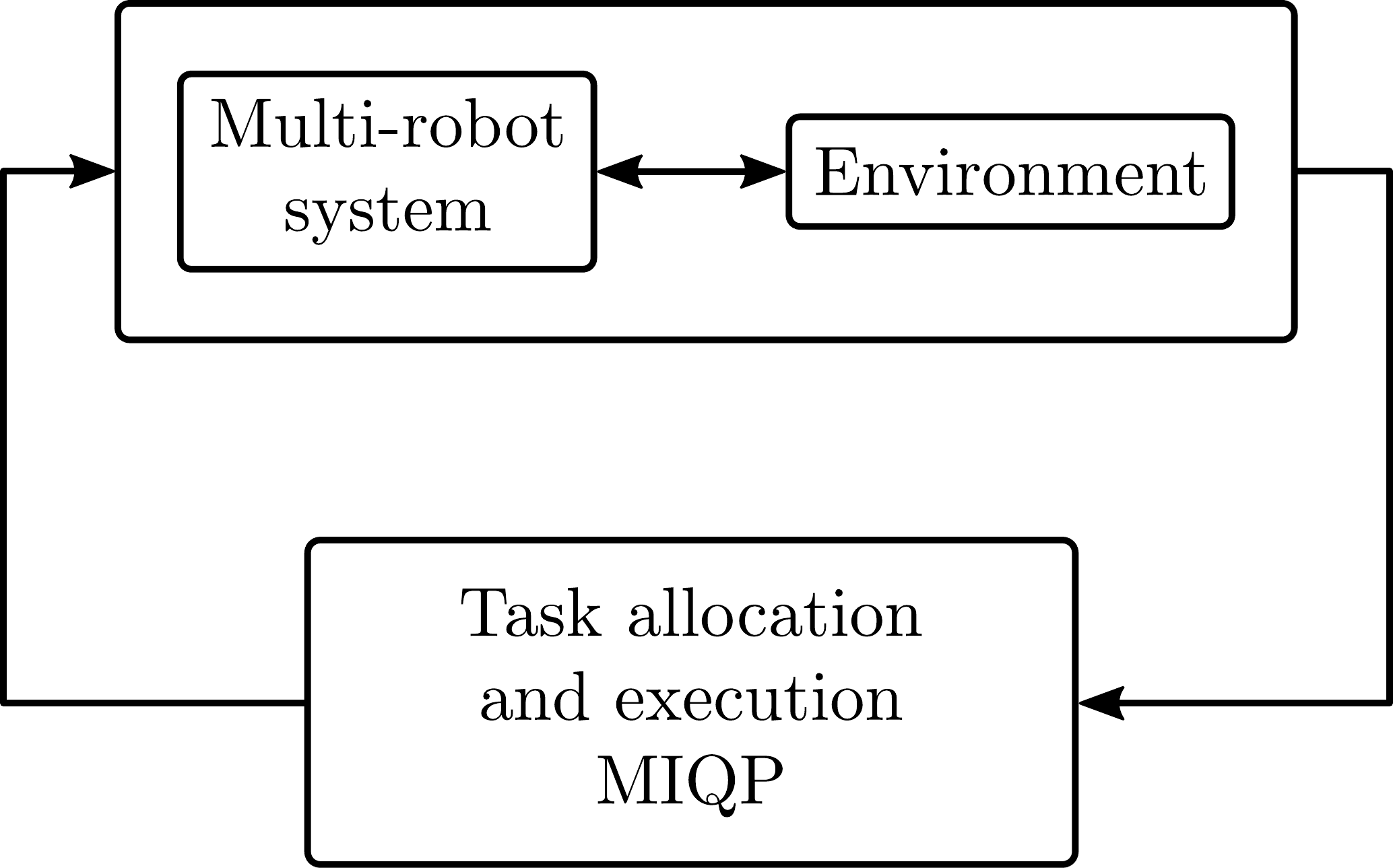}
\caption{The multi-robot system interacting with the environment controlled in feedback by the task allocation and execution optimization program \eqref{eq:allocationalgorithmactual}.}
\label{fig:lure}
\end{figure}

In this subsection, we introduce two distinct methods that render the task allocation and execution framework presented above resilient to environmental disturbances and robot feature failures. To achieve this, we leverage the fact that the optimization problem presented in~\eqref{eq:allocationalgorithmactual} is solved point-wise in time, and thus can be integrated along with online updates to the specializations and capabilities of the robots to construct a feedback loop as depicted in Fig.~\ref{fig:lure}.

We begin by introducing an update law aimed at changing the specialization values of the robots based on their measured versus expected progress at completing the tasks they are allocated to. The latter allows the framework to account for exogenous disturbances, i.e., disturbances that are not detectable, or not explicitly modeled, or just unknown \cite{sira2018active}. In the cases when the disturbances are endogenous (e.g., detectable sensor malfunction), we can directly account for them by modifying the specific values in the mappings introduced in Section~\ref{sec:encRobotHeter}.

\subsubsection{Exogenous disturbances} 
In certain deployment scenarios, the specializations of robots towards the tasks might be unknown prior to deployment of the team, or might vary due to changes in the  environmental  conditions. For the remainder of this paper, we refer to all such disturbances that cannot be modeled (i.e. cannot be accounted for in $F$) as \emph{exogenous} disturbances. These also include undetectable failures of robot components which are only reflected, and therefore be detected, in the way the robot executes the assigned task. In these cases, we would like to update the specialization parameters $s_{ij}$ on-the-fly to account for such changes. As described in \cite{emam2020adaptive}, this is achieved through updating the parameters $s_{ij}$ at each time step $k$ based on the difference between the expected and actual effectiveness of the task allocation and execution framework, where we assume that this difference manifests itself in terms of variations in the dynamical model of the robot. At discretized time intervals $k\Delta t, k \in \mathbb{N}$, let $x^{(k)}_{\text{act}}$ denote the actual ensemble state of the multi-robot system and $^{(i)}x^{(k)}_{\text{sim}}$ the ensemble state simulated by robot $i$ assuming it itself obeyed its nominal dynamics with all the other robots being stationary. The simulated states can be then evaluated as follows:
\begin{equation}
\label{eq:x_sim}
^{(i)}x_{\text{sim},j}^{(k)} = \begin{cases}
x_{\text{act},i}^{(k-1)} + \Delta x^{(k-1)}_i \Deltat & \text{if } j = i\\
x_{\text{act},j}^{(k-1)} & \text{if } j \neq i,
\end{cases}
\end{equation}
where $^{(i)}x_{\text{sim},j}^{(k)}$ denotes the $j$-th component of $^{(i)}x_{\text{sim}}^{(k)}$, and $\Delta x^{(k-1)}_i$ is defined as
\begin{equation}
\Delta x^{(k-1)}_i = f\left(x_{\text{act},i}^{(k-1)}\right) + g\left(x_{\text{act},i}^{(k-1)}\right)u_i^{(k-1)},
\end{equation}
$u_i^{(k-1)}$ being the input evaluated by solving \eqref{eq:allocationalgorithmactual} at time $(k-1)\Delta t$.
Using $^{(i)}x^{(k)}_{\text{sim}}$, robot $i$ can measure its contribution towards the difference between the simulated and the actual progress in the completion of task $m$ at time step $k$ as follows:
\begin{equation}
\label{eq:deltah}
\Delta h_{im}^{(k)} = \min\left\{0, h_{im}\left(x_{\text{act}}^{(k)}\right) - h_{im}\left(^{(i)}x_{\text{sim}}^{(k)}\right)\right\},
\end{equation}
where $h_{im}\left(^{(i)}x_{\text{sim}}^{(k)}\right)$ and $h_{im}\left(x_{\text{act}}^{(k)}\right)$ are the simulated and actual values of the CBF corresponding to robot $i$ and task $m$ at time step $k$, respectively. Note that the $\min$ operator in \eqref{eq:deltah} is used to prevent $\Delta h_{im}^{(k)}$ from being positive. This situation may occur due to the coordinated nature of multi-robot tasks, where robot $i$ need not know the actions of its neighbors, which could result in an unpredictable positive variations of $h_{im}$.

We assume that the CBF corresponding to each task, $h_m$, is decomposable into the respective contributions of each robot $i$, $h_{im}$. This assumption holds for a large number of coordinated control tasks such as multi-robot coverage control and formation control~\cite{cortes2017coordinated}, and allows each robot to assess its own effectiveness at executing a task by measuring $\Delta h_{im}^{(k)}$. In fact, if $\Delta h_{im}^{(k)}<0$, robot $i$'s actual effectiveness at accomplishing task $m$ is lower than anticipated. Consequently, one can model the evolution of the specialization of robot $i$ at task $m$ according to the following update law:
\begin{equation}
    \label{eq:spUpdate1}
    s_{im}^{(k+1)} = s_{im}^{(k)} + \beta \alpha_{im}^{(k)} \Delta h_{im}^{(k)},
\end{equation}
where $\beta \in \R_{>0}$ is a constant controlling the update rate. Note that the update only occurs for tasks to which the robots are assigned since $\alpha_{im}^{(k)} = 1$ if and only if robot $i$ is assigned to task $m$ at time step $k$. This update law renders the framework resilient to unknown environmental disturbances by allowing the framework to account for the dynamical variations in the environmental conditions through the updates of the specialization matrix according to the performance of the robots. Although it is not in the scope of this paper, in \cite{emam2020adaptive} we also show conditions under which the robot specialization lost because of the update law in \eqref{eq:spUpdate1} can be recovered over time. Algorithm~\ref{alg:2} extends Algorithm~\ref{alg:1} developed in the previous section to account for exogenous disturbances.

\begin{algorithm}
	\caption{Task allocation and execution resilient to exogenous disturbance}
	\label{alg:2}
 	\begin{algorithmic}[1]
	\Require
		\Statex Tasks $h_m,~m\in\{1,\ldots,n_t\}$
		\Statex Mappings $F$, $T$
		\Statex Parameters $n_{r,m,\text{min}}$, $n_{r,m,\text{max}}$, $\delta_\text{max}$, $C$, $l$
	\State Evaluate $S_i,~\forall i\in\{1,\ldots,n_r\}$ \Comment \eqref{eq:specialization}
	\While{true}
		\State Get robot state $x_i,\forall i\in\{1,\ldots,n_r\}$
		\State Compute robot input $u_i,\forall i\in\{1,\ldots,n_r\}$ \Comment \eqref{eq:allocationalgorithmactual}
		\State Send input $u_i,~\forall i\in\{1,\ldots,n_r\}$ to robots to execute
		\ForAll{$i\in\{1,\ldots,n_r\}$}
			\State Evaluate simulated robot state $^{(i)}x_{\text{sim}}^{(k)}$ \Comment \eqref{eq:x_sim}
			\ForAll{$m\in\{1,\ldots,n_t\}$}
				\State Evaluate $\Delta h_{im}^{(k)}$ \Comment \eqref{eq:deltah}
				\State Evaluate $s_{im}^{(k+1)}$ \Comment \eqref{eq:spUpdate1}
			\EndFor
		\EndFor
	\EndWhile
 	\end{algorithmic}
\end{algorithm}

The following example showcases the use of Algorithm~\ref{alg:2} in a simplified scenario with 2 robots, 2 tasks, and an unmodeled, exogenous environmental disturbance.

\begin{example}
\label{ex:sim2}
\begin{figure}
    \centering
	\subfloat[]{\label{subfig:ex2:map}\includegraphics[width=0.22\textwidth]{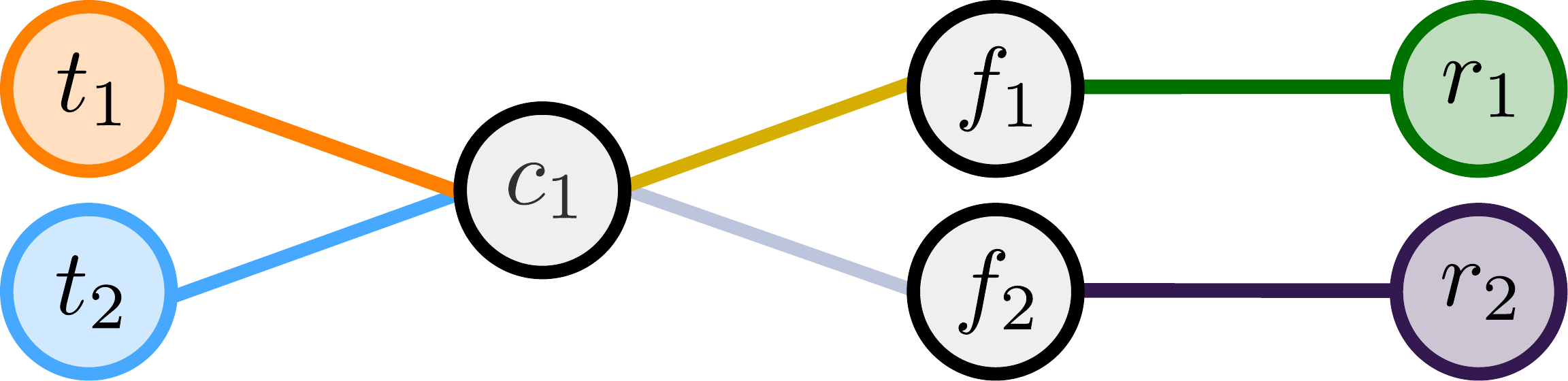}}\quad
    \subfloat[]{\label{subfig:ex2:traj}\includegraphics[width=0.2\textwidth]{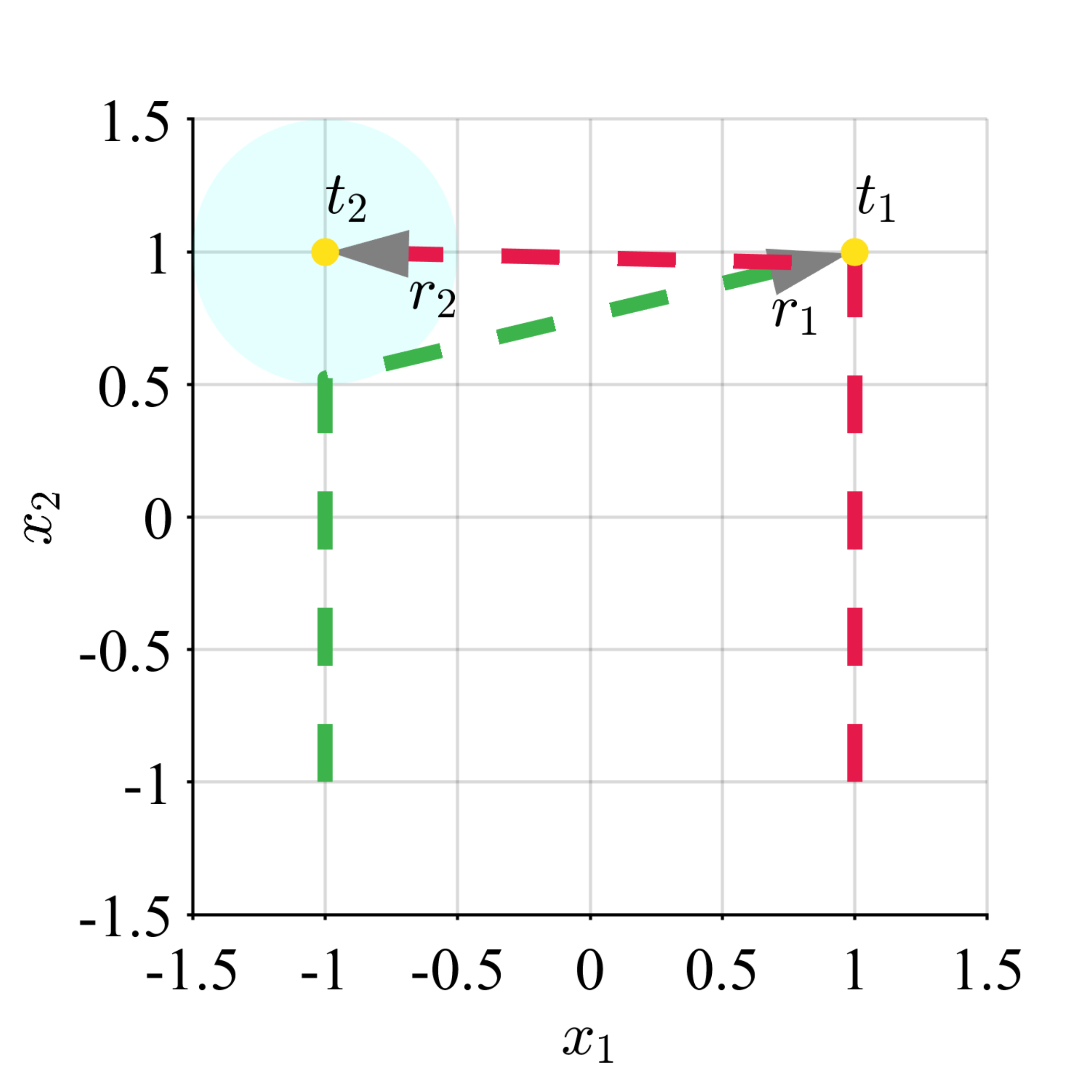}}\\
    \subfloat[]{\label{subfig:ex2:S}\includegraphics[width=0.2\textwidth]{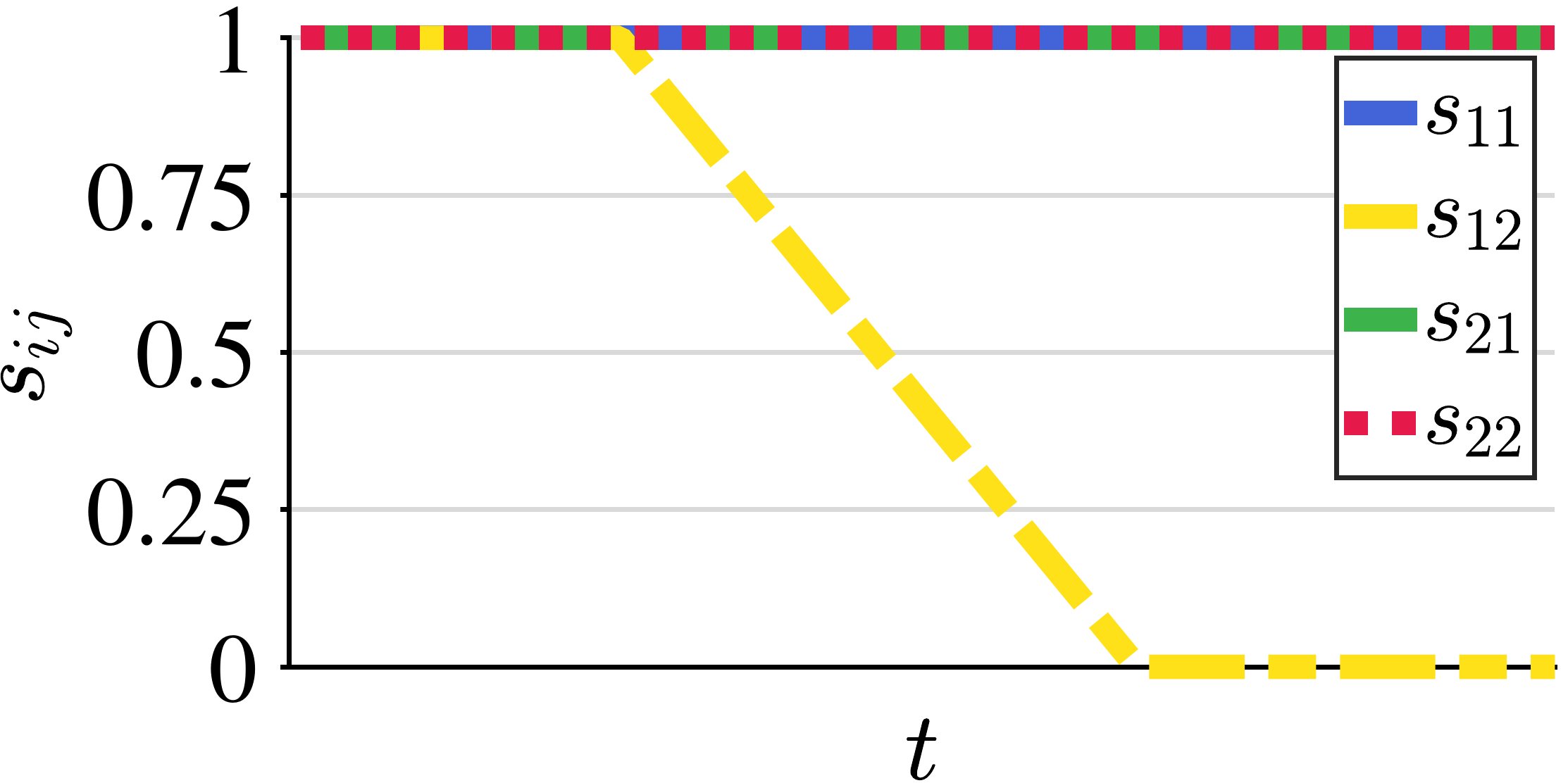}}
    \caption{Resilience of the task allocation algorithm to exogenous disturbances (Example~\ref{ex:sim2}). Robots $r_1$ and $r_2$ (gray triangles) have to execute tasks $t_1$ and $t_2$. They both possess the capabilities to perform both tasks (as pictorially shown in Fig.~\ref{subfig:ex2:map}), but $r_1$ is not capable of traversing the circular cyan-shaded region (representing the exogenous disturbance), rendering practically impossible for it to execute task $t_2$. Based on \eqref{eq:allocationalgorithmactual}, $r_1$ is initially assigned to $t_2$ and $r_2$ to $t_1$. As $r_1$ reaches the cyan zone, it is not able to proceed forward. According to Algorithm~\ref{alg:2}, per \eqref{eq:spUpdate1}, $s_{12}$---the specialization of $r_1$ to execute $t_2$, depicted in \protect\ref{subfig:ex2:traj} as a function of time $t$---starts decreasing until it reaches 0. At this point, the allocation evaluated by \eqref{eq:allocationalgorithmactual} automatically changes and robots $r_1$ and $r_2$ are assigned to tasks $t_1$ and $t_2$, respectively, fulfilling, this way, the requirement that both tasks need to be executed. The trajectories of the robots resulting by the allocation algorithm are depicted as dashed lines in Fig.~\protect\ref{subfig:ex2:traj}.}
    \label{fig:ex2}
\end{figure}
Consider the example depicted in Fig.~\ref{fig:ex2}. The 2 robots, $r_1$ and $r_2$ (shown as gray triangles, and modeled as 2-dimensional single integrators as in Example~\ref{ex:sim1}) are asked to execute 2 tasks $t_1$ and $t_2$. Their features and capabilities are depicted in Fig.~\ref{subfig:ex2:map}: both robots are capable of performing both tasks. Nevertheless, robot $r_1$ cannot traverse the region of the state space shaded in cyan (unmodeled disturbance), making the execution of task $t_2$ impossible for it. By implementing the control obtained by solving \eqref{eq:allocationalgorithmactual}, robot $r_1$ is initially assigned to task $t_2$, while $r_2$ is assigned to $t_1$, as confirmed by the initial vertical segments of the dashed green and red trajectories of the robots. As $r_1$ reaches the circular cyan region, it is not able to advance anymore and the execution of Algorithm~\ref{alg:2} makes its specialization towards task $t_2$---represented by $s_{12}$---decrease according to \eqref{eq:spUpdate1}, as depicted in Fig.~\ref{subfig:ex2:traj}. When $s_{12}=0$, the allocation algorithm \eqref{eq:allocationalgorithmactual} swaps the allocation of tasks as robot $r_1$ is not able to execute task $t_2$ to any extent anymore. The final allocation satisfies the requirements that both tasks are executed.
\end{example}

\subsubsection{Endogenous disturbances}
We now shift our focus to cases where the disturbances to the model are known to the robots---a condition happening in case of, e.g., detectable sensor malfunctions. We refer to this class of disturbances as \emph{endogenous} disturbances for which we account by directly altering the mappings introduced in Section~\ref{sec:encRobotHeter}. Specifically, by leveraging the feature representation, we directly alter the intermediate mappings (i.e. Robot-to-Feature and Feature-to-Capability mappings) on the fly to reflect such changes. The latter is achieved through modifying the corresponding values in the mappings defined in Section~\ref{sec:encRobotHeter} and re-computing the Robot-to-Capability matrix $F$ and the specialization matrices $S_i$. Note that $F$ is incorporated in the task allocation framework through the constraint \eqref{eq:miqp:f}, which ensures that the task allocation among the robots reflects the change in $F$. For example, in case of a feature failure, the Robot-to-Feature mapping matrix $A$ is altered to account for the failure. Moreover, in case of known environmental disturbances, the feature bundle weights $W_k$ is altered for each capability.  Following the example from Fig.~\ref{fig:quad}, if feature $f_4$ of robot $r_2$ malfunctions, we reflect that by altering the original $A$ matrix from \eqref{eq:a_matrix} to 
\begin{equation} \label{eq:a_matrix_altered}
A =
\begin{bmatrix}
     1 & 0 & 0 & 0 \\
     1 & 0 & 0 & 0 \\
     1 & 1 & 0 & 0 \\
     0 & 1 & 1 & 1 \\
     0 & 0 & 1 & 1 \\
     0 & 0 & 0 & 1 \\
\end{bmatrix},
\end{equation}
which is equivalent to removing the edge from robot $r_2$ to feature $f_4$ in the hypergraph from Fig.~\ref{fig:quad}. Similarly, known external environmental disturbances such as weather or terrain conditions are modeled by altering the weight vectors $w_k$ introduced in subsection~\ref{subsec:weightsExt}.

The approach described in this section to cope with endogenous disturbances is summarized in Algorithm~\ref{alg:3}. To conclude the section, we present a final example to showcase the behavior resulting from the application of Algorithm~\ref{alg:3}.

\begin{algorithm}
	\caption{Task allocation and execution resilient to endogenous disturbance}
	\label{alg:3}
 	\begin{algorithmic}[1]
	\Require
		\Statex Tasks $h_m,~m\in\{1,\ldots,n_t\}$
		\Statex Mappings $F$, $T$
		\Statex Parameters $n_{r,m,\text{min}}$, $n_{r,m,\text{max}}$, $\delta_\text{max}$, $C$, $l$, $\beta$
	\State Evaluate $S_i,~\forall i\in\{1,\ldots,n_r\}$ \Comment \eqref{eq:specialization}
	\While{true}
		\State Get robot state $x_i,~\forall i\in\{1,\ldots,n_r\}$
		\State Calculate robot input $u_i,~\forall i\in\{1,\ldots,n_r\}$ \Comment \eqref{eq:allocationalgorithmactual}
		\State Send input $u_i,~\forall i\in\{1,\ldots,n_r\}$ to robots to execute
		\State Update matrix $A$
		\State Re-evaluate matrices $F$ and $S$ \Comment \eqref{eq:featuresmapping}, \eqref{eq:specialization}
	\EndWhile
 	\end{algorithmic}
\end{algorithm}

\begin{example}
\label{ex:sim3}
\begin{figure}
    \centering
	\subfloat[]{\label{subfig:ex3:map}\includegraphics[width=0.22\textwidth]{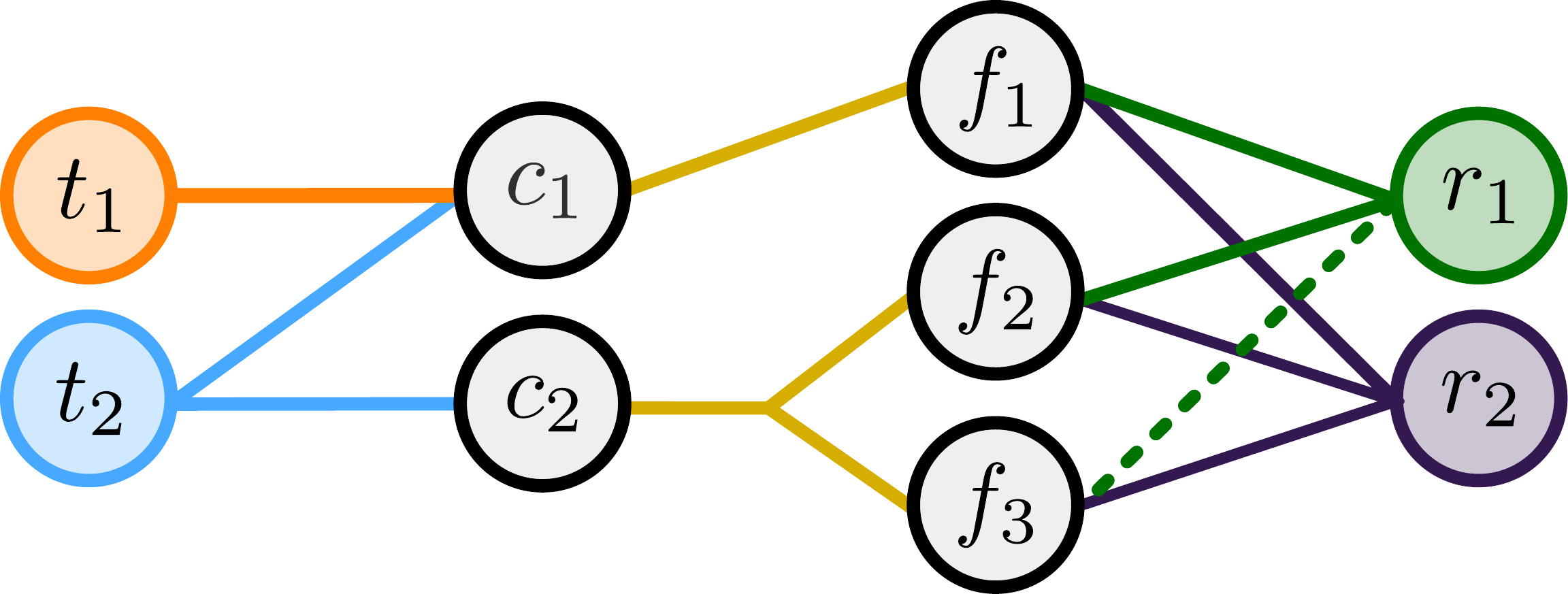}}\quad
	\subfloat[]{\label{subfig:ex3:traj}	\includegraphics[width=0.2\textwidth]{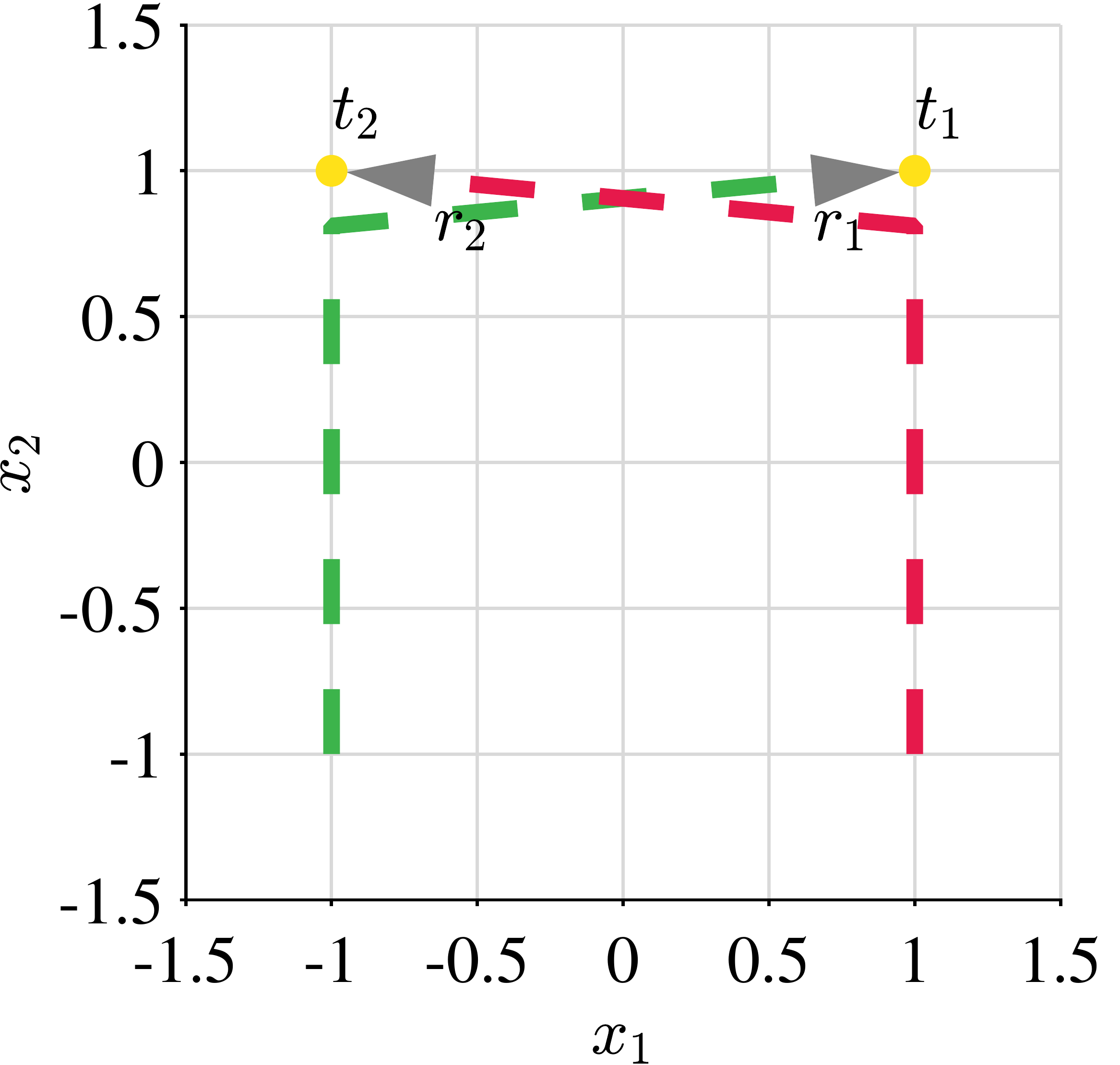}}
    \caption{Resilience of the task allocation algorithm to endogenous disturbances (Example~\ref{ex:sim3}). The 2 robots $r_1$ and $r_2$ (gray triangles) are asked to perform tasks $t_1$ and $t_2$. Initially, both robots are able to perform both tasks based on their possessed features shown in Fig.~\protect\ref{subfig:ex3:map}. Solving \eqref{eq:allocationalgorithmactual} initially assigns robot $r_1$ to task $t_2$ and $r_2$ to $t_1$. At a certain time instant, $A_{31}$ transitions from 1 to 0, corresponding to the condition that robot $r_1$ loses feature $f_3$ (the dashed edge in Fig.~\protect\ref{subfig:ex3:map} is lost). At this point, the constraint \eqref{eq:miqp:e} forces the task allocation to swap so that $r_2$, the only robots capable of providing capability $c_2$ for executing $t_2$, is assigned to it. This way, the requirement that both tasks are executed by at least one robot are satisfied. The trajectories of the robots are depicted as dashed lines in Fig.~\protect\ref{subfig:ex3:traj}.}
    \label{fig:ex3}
\end{figure}
In this last examples, 2 robots, $r_1$ and $r_2$ are considered, which possess the features depicted in Fig.~\ref{subfig:ex3:map} to execute 2 tasks, $t_1$ and $t_2$, thanks to the capabilities $c_1$ and $c_2$. The mappings from features to capabilities to tasks may represent the following scenario. Two robots are endowed with wheels (feature $f_1$) for mobility (capability $c_1$), as well as a camera (feature $f_2$) and a communication module (feature $f_3$) serving the ability of live streaming (capability $c_2$). Task $t_1$ consists in visiting a location of the state space of the robots, while task $t_2$ consists in visiting a location and live streaming a video feed. The endogenous disturbance consists in robot $r_1$ losing the communication functionality at a certain time instant, compromising its ability of performing task $t_2$, as it cannot live stream video feed anymore. The dashed edge in Fig.~\ref{subfig:ex3:map} is lost, and the robot-to-feature mapping matrix $A$ is modified by setting $A_{31}=0$.

Fig.~\ref{subfig:ex3:traj} depicts the trajectories of the robots under the initial allocation of robot $r_1$ to task $t_2$ and $r_2$ to $t_1$. As $A_{31}=0$, the matrix $F$ changes according to \eqref{eq:cap2agent}. Consequently, the constraint \eqref{eq:miqp:e} in the optimization problem \eqref{eq:allocationalgorithmactual} prevents robot $r_1$ from being allocated to task $t_2$. Thus, the task allocation swaps in order to be able to perform both tasks, as required.
\end{example}

\section{Analysis and Implementation of the Task Prioritization and Execution Algorithm} \label{sec:mixedc-dec}

The definition of the optimization problem as in \eqref{eq:allocationalgorithmactual} gives rise to two main questions: (i) whether, despite its pointwise-in-time nature, the allocation algorithm generates a stable allocation of tasks among robots, and (ii) whether it can be solved in real time to allocate tasks to robots and synthesize control inputs which allow robots to execute them. As far as (i) is concerned, a stable allocation is the amenable condition under which, with time-invariant parameters of the problem and no exogenous or endogenous disturbances, each robot does not continuously switch between the tasks it executes, but rather is able to accomplish one of them. Regarding (ii), \eqref{eq:allocationalgorithmactual} is a mixed-integer quadratic program and, as such, solving it in real time might too computationally intensive.

To address these two issues, in Section~\ref{subsec:convergence} we present results on the convergence of the task prioritization and execution algorithm introduced in the previous section. These results guarantee that the allocation of tasks to a heterogeneous multi-robot system obtained by executing Algorithm~\ref{alg:1} will converge, allowing the robots to complete the tasks that have been assigned to them. Moreover, in Section~\ref{subsec:mixed}, we present a mixed centralized/decentralized implementation of the developed task allocation algorithm which enables its application in online settings.

\subsection{Analysis of Convergence of the Task Prioritization and Execution Algorithm}
\label{subsec:convergence}

In cases where the tasks that the robots are asked to execute are neither coordinated nor time-varying (namely the CBF associated with them does not explicitly depend on the time variable), the following Proposition shows that the application of the task allocation algorithm \eqref{eq:allocationalgorithmactual} leads to a convergent behavior of the robots, whose states converge to a stable equilibrium point.

\begin{proposition}
\label{prop:particular}
Consider $n_r$ robots modeled by the driftless dynamical system
\begin{equation}
\label{eq:driftless}
\dot x_i = g(x_i) u_i,
\end{equation}
executing the control input $u_i^{(k)}$ at time $k$, where $u^{(k)}$ is obtained by solving the task allocation algorithm \eqref{eq:allocationalgorithmactual} at time $k$ in order to perform $n_t$ tasks. Assume that the tasks are characterized by the functions $h_1,\ldots,h_{n_t}$ which do not have an explicit dependence on time. Assume further that the tasks are not coordinated, i.e.
\begin{equation}
h_m(x) = \sum_{i=1}^{n_r} h_{m,i}(x_i)\quad\forall i\in\{1,\ldots,n_t\},
\end{equation}
where
\begin{equation}
\label{eq:gradientclassk}
\left\|\frac{\partial h_{m,i}(x_i)}{\partial x_i}\tr(x_i) \right\| =  \lambda\left(h_{m,i}(x_i)\right),
\end{equation}
$\lambda$ being a class $\mathcal{K}$ function, and there exist a unique $x_{m,i}^\star$---corresponding to the state at which the task characterized by the function $h_{m,i}$ is completed---such that $h_{m,i}(x_{m,i}^\star)=0$. If all robots are capable of performing all tasks to a certain extent, i.e. the matrices $S_i$, $i\in\{1,\ldots,n_r\}$ are positive definite, then the sequences $\{u^{(k)}\}_{k\in\mathbb N}$, $\{\delta^{(k)}\}_{k\in\mathbb N}$, and $\{\alpha^{(k)}\}_{k\in\mathbb N}$, solutions of \eqref{eq:allocationalgorithmactual}, converge as $k\to\infty$. In particular, $u^{(k)}\to0$, $\delta^{(k)}\to0$.
\end{proposition}
\begin{proof}[Proof]
Solving \eqref{eq:allocationalgorithmactual} at time $k$ yields $u^{(k)}$, $\delta^{(k)}$, and $\alpha^{(k)}$.
At time $k+1$, by Proposition~3 in \cite{notomista2019optimal} where $\alpha=\alpha^{(k)}$ and $J_m(x)=-h_{m}(x)$, if $\alpha^{(k+1)}=\alpha^{(k)}$ and $\delta^{(k+1)}=\delta^{(k)}$, then $\|u^{(k+1)}\|<\|u^{(k)}\|$ is obtained using \eqref{eq:gradientclassk}.
Let
\begin{equation}
V(\alpha,u,\delta) = \sum_{i = 1}^{n_r} \left( C \| \Pi_i \alpha_{-,i}\|^2 +  \|u_i\|^2 + l \|\delta_i \|_{S_i}^2 \right),
\end{equation}
be a candidate Lyapunov function for the multi-robot system controlled via the solutions of the optimization problem~\eqref{eq:allocationalgorithmactual} (see scheme in Fig.~\ref{fig:lure}). Notice that $V$ is equal to the cost \eqref{eq:miqp:a} and it is positive definite since $S_i$ is positive definite for all $i$ by assumption. Then, one has:
\begin{align}
V^{(k+1)} &= V(\alpha^{(k+1)},u^{(k+1)},\delta^{(k+1)})\\
&\le V(\alpha^{(k)},u^{(k+1)},\delta^{(k)})\\
&< V(\alpha^{(k)},u^{(k)},\delta^{(k)})=V^{(k)}.
\end{align}
Therefore, $V^{(k)}\to0$ as $k\to\infty$. Thus, $u^{(k)}\to0$ as $k\to\infty$, and $x_i^{(k)}\to x_{m,i}^\star$ for some $m$, by the driftless assumption on the robot model \eqref{eq:driftless}.
\end{proof}

The application of the previous results is however restricted to a specific, but nevertheless quite rich, class of tasks. However, in situations where the assumptions of Proposition~\ref{prop:particular} are not satisfied, the following proposition provides us with sufficient conditions to ensure the convergence of the flow of the dynamical system comprised of the multi-robot system, characterized by its nonlinear dynamics, in feedback with the optimization problem embodying the task allocation algorithm (depicted in Fig.~\ref{fig:lure}). The similarity between the system in Fig.~\ref{fig:lure} and the Lure's problem \cite{khalil2002nonlinear} suggests us to resort to techniques that have been widely adopted in the stability analysis of such systems, with the aim of studying the convergence of the task allocation algorithm we propose in this paper. Indeed, in the following proposition a quadratic Lyapunov function is proposed, and conditions to establish the convergence of the task allocation algorithm are given in the form of a linear matrix inequality (LMI) using the S-procedure \cite{boyd1994linear,sprocedure}.

\begin{proposition}
\label{prop:general}
If, for all integers $k$, there exist positive scalars $\tau_1$, $\tau_2$, $\tau_3$ such that
\begin{equation}
\label{eq:lmi}
B_0^{(k)} \le \tau_1 B_1^{(k)} + \tau_2 B_2^{(k)} + \tau_3 B_3^{(k)},
\end{equation}
where $B_0^{(k)}$, $B_1^{(k)}$, $B_2^{(k)}$, $B_3^{(k)}$ are given by \eqref{eq:F} in Appendix~\ref{app:eq:F}, in which $c \in \R_{>0}$, then the sequences $\{u^{(k)}\}_{k\in\mathbb N}$, $\{\delta^{(k)}\}_{k\in\mathbb N}$, and $\{\alpha^{(k)}\}_{k\in\mathbb N}$, solutions of \eqref{eq:allocationalgorithmactual} at time step $k$, converge as $k\to\infty$.
\end{proposition}
\begin{proof}[Proof]
For notational convenience, we let the $\bar\alpha = \begin{bmatrix}
\alpha_{-,1}\tr&\ldots&\alpha_{-,n_r}\tr
\end{bmatrix}\tr\in\{0,1\}^{n_t n_r}$ be the vector composed of the stacked columns of $\alpha$, and
\begin{equation}
\bar\Phi = \mathbbm{1}_{n_r} \otimes \Phi, \quad
\bar\Theta = \mathbbm{1}_{n_r} \otimes \Theta, \quad
\bar\Psi = \mathbbm{1}_{n_r} \otimes \Psi,
\end{equation}
$\otimes$ denoting the Kronecker product. From \eqref{eq:deltaalphaconstraint} and with the notation introduced above, one has that 
\begin{equation}
\bar\Phi \bar\alpha \ge 0,
\end{equation}
where the symbol $\ge$ is always intended component-wise. Then, as $\delta\in\R_{\ge0}$ (see discussions in \cite{notomista2019constraint} and \cite{notomista2019optimal}), the constraints \eqref{eq:miqp:b} and \eqref{eq:miqp:c} in \eqref{eq:allocationalgorithmactual} can be re-written as follows:
\begin{align}
&\begin{aligned}
&\delta^{(k)T} L_f h(x^{(k)}) + \delta\tr L_g h(x^{(k)}) u^{(k)}\\ \geq &-\delta^{(k)T} \gamma(h(x^{(k)})) - \delta^{(k)T} \delta^{(k)}
\end{aligned}\label{eq:mixedconstraints1}\\[0.2cm]
&\bar\alpha\tr\bar\Phi\tr \bar\Theta\bar\delta^{(k)} + \bar\alpha\tr\bar\Phi\tr \bar\Phi\bar\alpha^{(k)} \le \bar\alpha\tr\bar\Phi\tr \bar\Psi \label{eq:mixedconstraints2}.
\end{align}
Similarly, the constraints \eqref{eq:miqp:d} to \eqref{eq:miqp:g}, can be re-written as
\begin{align}
&\bar\alpha^{(k)T} A_\alpha\tr A_\alpha \bar\alpha^{(k)} \le \bar\alpha^{(k)T} A_\alpha\tr b_\alpha \label{eq:boundsconstraintsalpha}\\
&\delta^{(k)T} A_\delta\tr A_\delta \delta^{(k)} \le \delta^{(k)T} A_\delta\tr b_\delta \label{eq:boundsconstraintsdelta}.
\end{align}

Then, define the following candidate Lyapunov function:
\begin{equation}
\label{eq:lyapunov}
V(x) = \gamma(h(x))\tr \gamma(h(x)),
\end{equation}
where $h(x)=[h_1(x),\ldots,h_{n_t}(x)]\tr$ and $\gamma(h(x))$ is intended as a component-wise application of the extended class $\mc K_\infty$ function to the vector $h(x)$. We want the following condition on its time derivative to be satisfied at every time step $k$
\begin{align}
\label{eq:lyapunovfunctioncondition}
\dot V(x^{(k)},u^{(k)}) &= 2 \gamma(h(x^{(k)}))\tr\frac{\mathrm{d}\gamma}{\mathrm{d}h}\frac{\mathrm{d}h}{\mathrm{d}x}f(x^{(k)}) \\
&+ 2 \gamma(h(x^{(k)}))\tr\frac{\mathrm{d}\gamma}{\mathrm{d}h}\frac{\mathrm{d}h}{\mathrm{d}x}g(x^{(k)}) u^{(k)} \\
&\le -c V(x^{(k)}),
\end{align}
with $c \in \R_{>0}$.

Defining $\varphi^{(k)}=[\gamma(h(x^{(k)})), u^{(k)}, \delta^{(k)}, \bar\alpha^{(k)}, 1]\tr$, the inequalities \eqref{eq:lyapunovfunctioncondition},  \eqref{eq:mixedconstraints1}, \eqref{eq:mixedconstraints2}, \eqref{eq:boundsconstraintsalpha}, \eqref{eq:boundsconstraintsdelta} can be compactly written as follows:
\begin{gather}
\varphi^{(k)T} B_0^{(k)} \varphi^{(k)} \le 0\\
\varphi^{(k)T} B_1^{(k)} \varphi^{(k)} \le 0\\
\varphi^{(k)T} B_2^{(k)} \varphi^{(k)} \le 0\\
\varphi^{(k)T} B_3^{(k)} \varphi^{(k)} \le 0,
\end{gather}
where $B_0$, $B_1$, $B_2$, and $B_3$ are defined in \eqref{eq:F}.

Thus, applying the S-procedure \cite{boyd1994linear}, the linear matrix inequality \eqref{eq:lmi} in the variables $\tau_1$, $\tau_2$, $\tau_3$ is obtained. If a solution to \eqref{eq:lmi} exists for all $k$, then \eqref{eq:lyapunovfunctioncondition} is satisfied for all $k$, and therefore $V(x^{(k)})\to0$. Consequently $x^{(k)}$ converges as $k\to\infty$, and so do the sequences $\{u^{(k)}\}_{k\in\mathbb N}$, $\{\delta^{(k)}\}_{k\in\mathbb N}$, and $\{\alpha^{(k)}\}_{k\in\mathbb N}$, solution of \eqref{eq:allocationalgorithmactual}, parameterized by $x^{(k)}$.
\end{proof}

\begin{remark}[Certificate of feasibility of \eqref{eq:lmi}]
The convergence of the multi-robot system executing the allocated tasks as shown in Proposition~\ref{prop:general} hinges on the existence of solution of \eqref{eq:lmi}. In \cite{boyd1994linear}, necessary and sufficient conditions for the existence of solutions are provided. For instance, for robots modeled with linear systems and tasks modeled with quadratic functions $h_m$, stability of the task allocation can be certified by solving an algebraic Riccati equation.
\end{remark}

Despite the flexibility determined by the variety of scenarios encompassed by the optimization-based task allocation formulation presented in this section, its mixed-integer nature does not allow, in most cases, to scale its applicability to a large number of robots \cite{lee2011mixed}. Therefore, it is not always possible to solve the proposed task allocation optimization program \eqref{eq:allocationalgorithmactual} in an online fashion under real-time constraints. Thus, in the following section, we propose a mixed centralized/decentralized execution strategy which allows the computation of the task prioritization as well as the control inputs required by the robots to execute the tasks to take place in online settings.

\subsection{Mixed Centralized/Decentralized Implementation of the Task Prioritization and Execution Algorithm}
\label{subsec:mixed}

\begin{figure}
    \centering
    \includegraphics[width=0.32\textwidth]{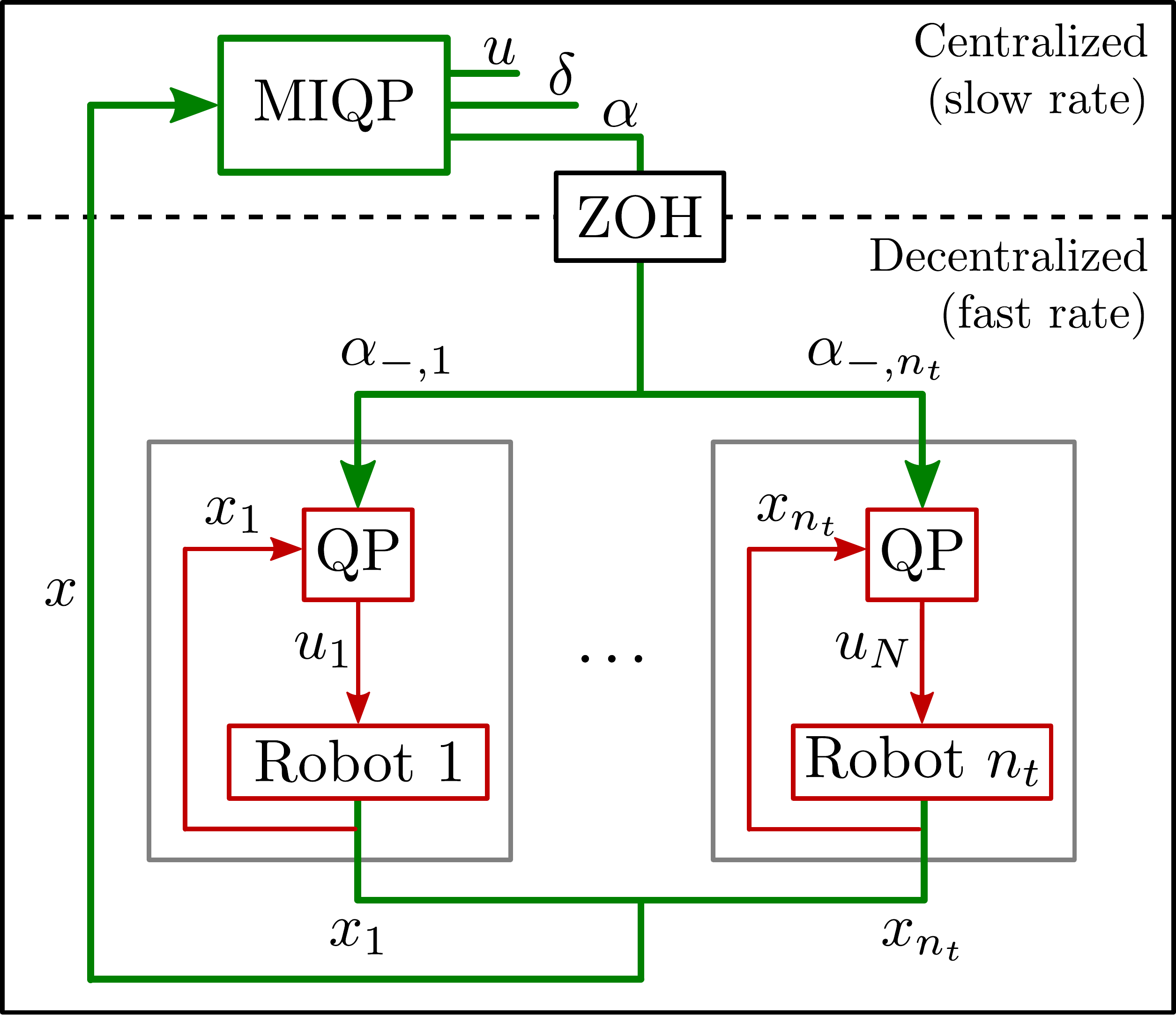}
    \caption{A mixed centralized/decentralized architecture to implement the task allocation and execution algorithm. Unlike the MIQP centralized formulation in \eqref{eq:allocationalgorithmactual}, the allocation is solved separately from the execution. The former is evaluated in a centralized fashion based on the states collected from all the robots, and it typically happens at a slower rate due to the computational complexity of mixed-integer programs. The latter is solved by each robot in a decentralized way once the allocation (in terms of $\alpha_{-,i}$) is received by the robots from the central computational unit. At the interface between slow and fast rate a zero-order hold block signifies that each robot $i$ receives a new allocation vector $\alpha_{-,i}$ each time this is obtained by solving the MIQP.}
    \label{fig:mixedcentralizeddecentralized}
\end{figure}

In order to allow the applicability of the proposed task prioritization and execution algorithm to scenarios where a large number of robots have to execute a large number of tasks, in the following we propose an alternative mixed centralized/decentralized formulation. We then analyze the performance in terms of task allocation and execution compared to the MIQP developed in the previous section.

To this end, the optimization problem \eqref{eq:allocationalgorithmactual} is solved by a central computational unit for $u$, $\delta$, and $\alpha$. The central computational unit then communicates to each robot $i$ only its allocation vector $\alpha_{-,i}$. At this point, each robot can solve the following convex quadratic program (QP) in order to compute the control input it requires to execute the task prioritized based on its prioritization vector $\alpha_{-,i}$ received by the central computational unit:

\begin{subequations} \label{eq:executionalgorithm}
\begin{flalign}
\text{\bf Task execution optimization problem (QP)} \tag{\ref{eq:executionalgorithm}}&&\label{eq:executionalgorithmactual}
\end{flalign}
\vspace{-0.75cm}
\begin{align}
\minimize_{u_i,\delta_i} & \|u_i\|^2 + l \|\delta_i \|_{S_i}^2 \label{eq:qp:a}\\
\subjto & L_f h_{m}(x) + L_g h_{m}(x) u_i \\
&\qquad\geq -\gamma(h_{m}(x)) - \delta_{im} \label{eq:qp:b} \\
&\Theta\delta_i + \Phi\alpha_{-,i}  \le \Psi \label{eq:qp:c}\\
&\|\delta_i\|_\infty \leq \delta_\text{max} \label{eq:qp:d}\\
&\hspace{2cm}\forall m \in \{1\ldots n_t\}.
\end{align}
\noeqref{eq:qp:a}\noeqref{eq:qp:b}\noeqref{eq:qp:c}\noeqref{eq:qp:d}
\end{subequations}
\hspace{-0.32cm} Depending on the coordinated nature of the tasks, the solution of \eqref{eq:executionalgorithm} can be obtained with or without communication between the robots. See \cite{notomista2019constraint} for a detailed discussion on how to achieve a coordinated control of multi-robot systems using this formulation.
Figure~\ref{fig:mixedcentralizeddecentralized} summarizes the described mixed centralized/decentralized architecture.

Notice that, if solving the centralized MIQP cannot be done at each time step, by following the mixed centralized/decentralized approach, each robot solves for its control input $u_i$ using an outdated value of its prioritization vector $\alpha_{-,i}$, which is calculated by the central unit using old values of the state $x_i$ of the robots. Depending on the time that the central computational unit takes to solve the MIQP, the difference between the input $u_i$ solution of \eqref{eq:executionalgorithm} and the one that would have been obtained by solving \eqref{eq:allocationalgorithmactual} might be different. In the following, we quantify the error that is introduced in the control input $u_i$ by adopting the mixed centralized/decentralized approach, rather than solving the centralized MIQP at each time step. 

For notational convenience, we introduce the following mappings. We denote by
\begin{equation}
\GammaMIQP \colon\R^{n_x n_r}\to\{0,1\}^{n_t\times n_r} \colon x \mapsto \alpha
\end{equation}
the natural projection of the solution map of \eqref{eq:allocationalgorithmactual}, and by
\begin{equation}
\GammaQP \colon\R^{n_x n_r}\times\{0,1\}^{n_t\times n_r}\to\R^{n_u n_r} \colon (x,\alpha) \mapsto u,
\end{equation}
the natural projection of the solution map of \eqref{eq:executionalgorithm} for all the robots. Moreover, we let $\Gamma(\cdot,\cdot) = \GammaQP(\cdot, \GammaMIQP(\cdot))$, and denote by $\GammabarMIQP$ the solution map of the QP relaxation of \eqref{eq:allocationalgorithmactual} projected onto the subspace of allocation vectors $\alpha$---where $\alpha\in[0,1]^{n_t\times n_r}\subset\R^{n_t\times n_r}$.

Assume that, at time $k \Deltat$, the central unit receives $\xk$ from the $n_r$ robots, and solves the MIQP \eqref{eq:allocationalgorithmactual} obtaining $\alphak=\GammaMIQP\left(\xk\right)$. This computation is assumed to take $n$ steps \footnote{If stopping criteria are not met, the algorithm times out after $n\Deltat$ seconds.}, or $n\Deltat$ seconds. At time $(k+n)\Deltat$, the central unit transmits the computed allocation values $\alphak$ to the robots, each of which solves \eqref{eq:executionalgorithm}, and the input to the robots can be expressed as $\ukn=\GammaQP\left(\xkn,\alphak\right)$. This is assumed to take 1 step, or $\Deltat$ seconds.
We are interested in quantifying the difference between the robot control inputs $\ukn$ evaluated by the robots with the old value $\alphak$ and the control input $\hatukn$ that would be evaluated with the current value $\alphakn$. This difference is given by \eqref{eq:difference}, and the different contributions are explicitly broken down in \eqref{eq:udiff} in Appendix~\ref{app:eq:udiff}, using the sensitivity results in \cite{granot1990some}. The notation $\Delta(A)$ in \eqref{eq:udiff} denotes the maximum of the absolute values of the determinants of the square submatrices of the matrix A. Moreover, $\mathscr A$ and $\mathscr B$ denote the matrix and the vector such that the inequality constraints in \eqref{eq:allocationalgorithmactual} can be written as
\begin{equation}
\mathscr A \begin{bmatrix}
u\\
\delta\\
\alpha
\end{bmatrix} \le \mathscr B,
\end{equation}
and, provided that the conditions of Theorem~5.2 in \cite{fiacco1990sensitivity} hold, $\LMIQP$, $\LQP$, and $L_{\dot x}$ are the Lipschitz constants of the mappings $\GammaMIQP$, $\GammaQP$, and the robot dynamics \eqref{eq:controlaffine}, respectively.

As expected, the bound on $\left\|\ukn-\hatukn\right\|$ in \eqref{eq:udiff} is a monotonically increasing function of the number of optimization variables, of the values $\LQP$ and $\LMIQP$---which, in turn, depend on the parameters of the optimization problem \cite{fiacco1990sensitivity}---, of $L_{\dot x}$, and of $n\Deltat$, i.e. the time required by the central computational unit to solve the MIQP. In particular, the bound in \eqref{eq:udiff} is comprised of two terms: the first one depends on the mixed-integer nature of the allocation algorithm \eqref{eq:allocationalgorithmactual}, while the second one is due to the computation time that the central unit takes in order to solve the allocation optimization. The effect of the mixed-integer programming is the most critical one, as it is proportional to $n_t^2 n_r^3$, and vanishes only when the solution of the MIQP \eqref{eq:allocationalgorithmactual} is equal to that of its QP relaxation. The term depending on the computational time, instead, vanishes if the MIQP can be solved at each time step.

\begin{remark}[Communication delays]
Notice that the time to communicate the allocation solution to all the robots, if not negligible, can be added to the quantity $n\Delta t$ to account for the effects of communication delays in the execution of the allocated tasks.
\end{remark}

We conclude this section by summarizing the mixed centralized/decentralized implementation of the proposed task allocation optimization problem in Algorithm~\ref{alg:4}, which is combined with Algorithms~\ref{alg:2} and \ref{alg:3} to obtain an \emph{efficient} implementation of the optimal allocation and execution algorithm \emph{resilient to endogenous and exogenous disturbances}. This combination will be showcased in the next section, where the implementation of the developed allocation and execution algorithm on a real multi-robot platform is presented.

\begin{algorithm}
	\caption{Mixed centralized/decentralized implementation of task allocation and execution}
	\label{alg:4}
 	\begin{algorithmic}[1]
	\Require
		\Statex Tasks $h_m,~m\in\{1,\ldots,n_t\}$
		\Statex Mappings $F$, $T$
		\Statex Parameters $n_{r,m,\text{min}}$, $n_{r,m,\text{max}}$, $\delta_\text{max}$, $C$, $l$
	\State Evaluate $S_i,~\forall i\in\{1,\ldots,n_r\}$ \Comment \eqref{eq:specialization}
	\Procedure{Central computational unit}{}
		\While{true}
			\State Get robots' state $x_i,~\forall i\in\{1,\ldots,n_r\}$
			\State Calculate allocation $\alpha$ \Comment \eqref{eq:allocationalgorithmactual}
			\State Send allocation $\alpha_{-,i},~\forall i\in\{1,\ldots,n_r\}$ to robots
			\State Update matrices $S$ and $F$ if required \Comment Algs.~\ref{alg:2}, \ref{alg:3}
		\EndWhile
    \EndProcedure
	\Procedure{Robot $i$}{}
		\While{true}
			\State Receive allocation $\alpha_{-,i}$ if ready
			\State Calculate input $u_i$ and execute \Comment \eqref{eq:executionalgorithmactual}
		\EndWhile
    \EndProcedure
 	\end{algorithmic}
\end{algorithm}

\section{Experiments} \label{sec:exp}

\begin{figure}
\centering
\includegraphics[trim={13cm 2.5cm 10.5cm 1.5cm},clip,width=0.3\textwidth]{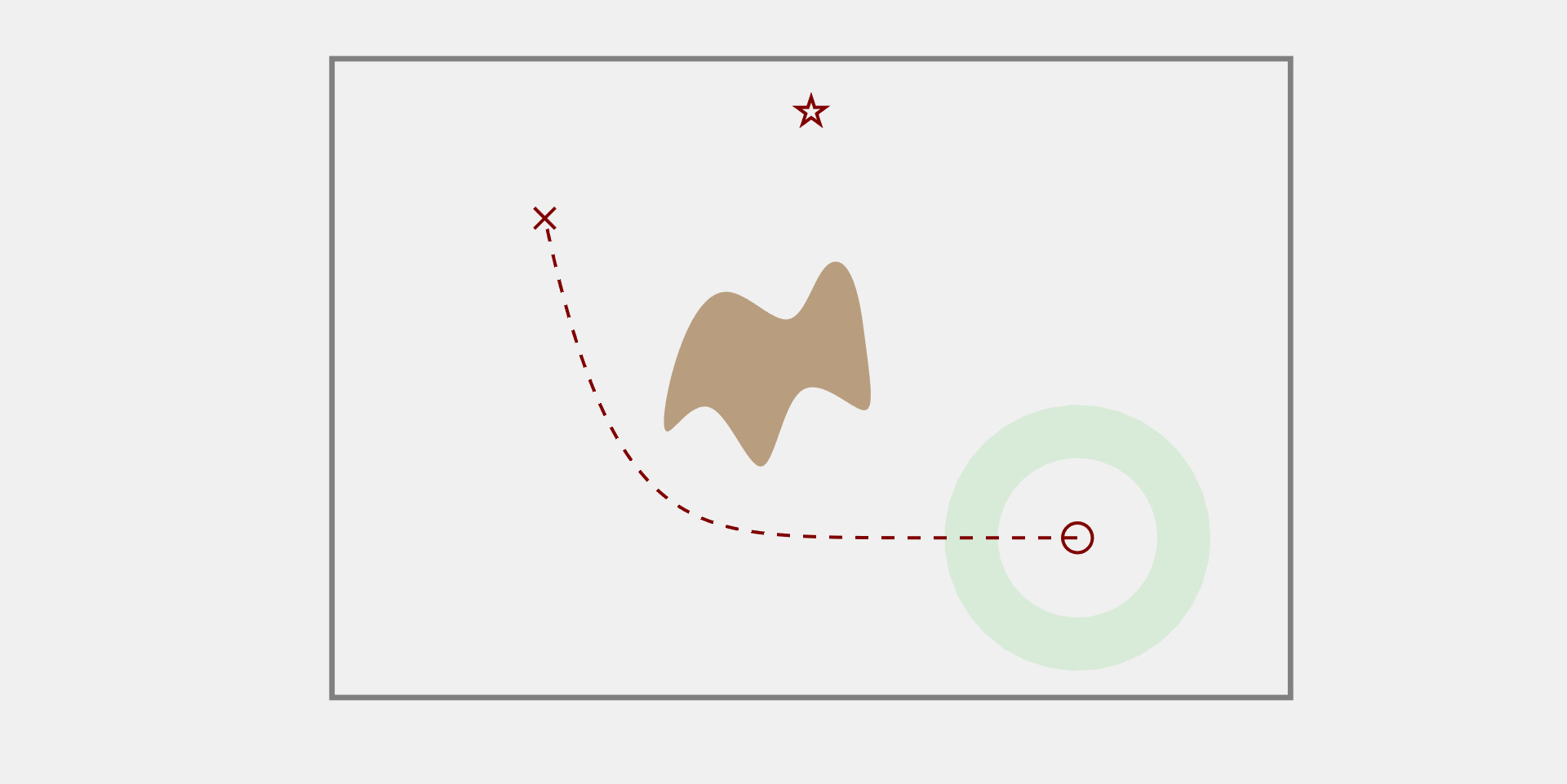}
\caption{Experimental scenario. The robots need to perform $2$ tasks: $1$ robot has to navigate in the environment to reach a goal point (red cross) following the dashed trajectory, while $3$ robots have to escort it by arranging themselves around it (on the green ring) while simultaneously monitoring a point of interest (red star). The brown blob in the middle of the rectangular environment represents a low-friction zone where the motion of ground robots is impeded.}
\label{fig:experiment:scenario}
\end{figure}

In order to illustrate the properties of the resilient task prioritization and execution framework developed and demonstrated in this paper, in this section we present the results of its implementation on a team of mobile robots in the Robotarium \cite{wilson2020robotarium}, a remotely accessible swarm robotics testbed. The scenario of the experiment is depicted in Fig.~\ref{fig:experiment:scenario}. A team of 5 mobile robots, each endowed with a simulated camera system, are deployed in a 3.6$\times$2.4 m rectangular domain and have to perform $2$ tasks: task $t_1$ consists of 1 robot moving along a desired trajectory navigating the environment from a starting point (red circle in Fig.~\ref{fig:experiment:scenario}) to a goal point (red cross in Fig.~\ref{fig:experiment:scenario}); to perform task $t_2$, $3$ robots need to escort the robot executing task $t_1$ by arranging themselves into a ring around it while simultaneously monitoring a point of interest with their cameras (red star in Fig.~\ref{fig:experiment:scenario}). The physical robots are differential drive robots. In the experiment, we model their motion as well as that of their cameras using the following single integrator dynamics:
\begin{equation}
\label{eq:robotcameramodel}
\dot x_{i,1} = u_1,\quad \dot x_{i,2} = u_2,\quad \dot x_{i,3} = u_3,
\end{equation}
where $p_i = [x_{i,1},x_{i,2}]\tr\in\R^2$ represents the position of robot $i$, and $x_{i,3}\in[0,2\pi]$ is the orientation of its camera. $u_1, u_2, u_3\in\R$ are the velocity inputs to the robot and to the camera.

Task $t_1$ is realized by tracking a predefined trajectory, while task $t_2$ is achieved by implementing a weighted coverage control \cite{cortes2004coverage,santos2019decentralized} in order to arrange the robots on the green ring in Fig.~\ref{fig:experiment:scenario}. The two tasks are encoded by the following two CBFs, respectively:
\begin{align}
&h_{1,i}(x,t) = - \| p_i - \hat p(t) \|^2\\
&h_{2,i}(x,t) = - \| p_i - G_i(x) \|^2 - \| x_{i,3} - \angle (p^*-p_i) \|^2,
\end{align}
where $\hat p : \R_{\ge0} \to \R^2$ is the desired trajectory (dashed line in Fig.~\ref{fig:experiment:scenario}), $p^*\in\R^2$ is the position of the point of interest to monitor (red star in Fig.~\ref{fig:experiment:scenario}), $\angle (p^*-p_i)$ denotes the angle formed by the vector $p^*-p_i$ with the horizontal coordinate axis, and $G_i(x)$ is the centroid of the Voronoi cell corresponding to robot $i$. In order to achieve the desired arrangement of robots performing task $t_2$ around the robot performing task $t_1$, the centroids $G_i$ have been evaluated as follows:
\begin{equation}
G_i(x) = \frac{\int_{\mc V_i} p_i \phi(p_i) \mathrm{d} p_i}{\int_{\mc V_i} \phi(p_i) \mathrm{d} p_i} \in\R^2
\end{equation}
where $\mc V_i$ is the Voronoi cell of robot $i$, $\phi(p_i)$ is the function
\begin{equation}
\phi(p_i) = e^{-k(\|p_i-\hat p(t)\|^2-r^2)^2},
\end{equation}
with $k \in \R_{>0}$ and $r$ being the radius of the green circle in Fig.~\ref{fig:experiment:scenario} (see \cite{cortes2004coverage} for details on coverage control). These two parameters have been set to $k=100$ and $r=0.4$.

\begin{figure}
\centering
\includegraphics[width=0.24\textwidth]{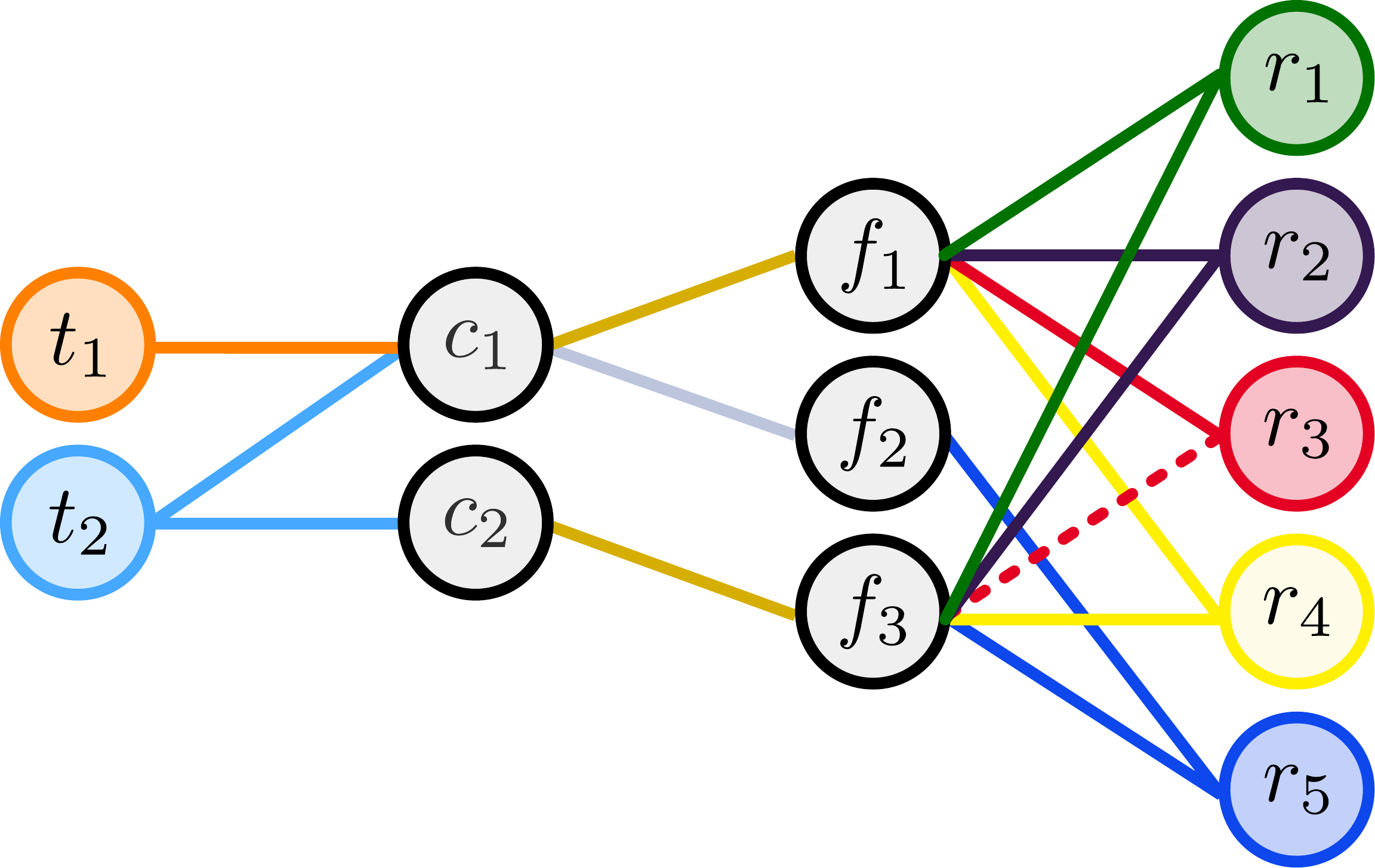}
\caption{Robots, features, capabilities and tasks mappings used for the experiment on the Robotarium. The features are wheels to locomote on the ground ($f_1$), propellers to locomote in the air ($f_2$), and a camera ($f_3$). The resulting capabilities are locomotion ($c_1$) and monitoring of a point of interest ($c_2$). Tasks consist of navigating the environment to reach a goal point ($t_1$), escorting the robot navigating the environment by arranging around it and monitoring a point of interest ($t_2$).}
\label{fig:experiment:map}
\end{figure}

Moreover, in order to be able to perform the prescribed tasks, the robots need certain features which allow them to exhibit the capabilities required by the two tasks. The mappings employed for the experiments are depicted in the bipartite graph in Fig.~\ref{fig:experiment:map}. The available features are wheels to locomote on the ground ($f_1$), set of propellers to fly ($f_2$), and a camera ($f_3$). The capabilities required to perform the given tasks are mobility ($c_1$) and monitoring ($c_2$). The former is supported by features $f_1$ or $f_2$, while the latter by $f_3$. To perform task $t_1$, only $c_2$ is required, while both capabilities are required for $t_2$. Finally, robots $r_1$ to $r_4$ are each endowed with wheels and a camera, while $r_5$ is able to fly---depicted in the Robotarium experiment by projecting down the shape of a quadcopter at its location---and possesses a camera.

Moreover, since 1 robot is required to be assigned to task $t_1$ and 3 robots to $t_2$ for all times, the following parameters have been set for the experiment:
\begin{equation}
T=\begin{bmatrix}
1&0\\
3&3
\end{bmatrix},\qquad\begin{gathered}
n_{r,1,\text{min}}=n_{r,1,\text{max}}=1,\\
n_{r,2,\text{min}}=n_{r,2,\text{max}}=3.
\end{gathered}
\end{equation}
Furthermore, the remaining parameters of \eqref{eq:allocationalgorithmactual} have been set to: $C = 10^6$, $l = 10^{-6}$, $\gamma \colon s \mapsto 5s$, $\kappa = 10^6$, $\delta_\text{max} = 10^3$. This choice was made considering the facts that (i) high values of $C$ result in the robot specialization to be respected as accurately as possible, and (ii) low values of $l$ result in the robots executing the assigned tasks as good as possible.

\begin{figure*}
\centering
\subfloat[][]{\label{subfig:exp:1}\includegraphics[trim={0cm 0cm 5cm 1cm},clip,width=0.245\textwidth]{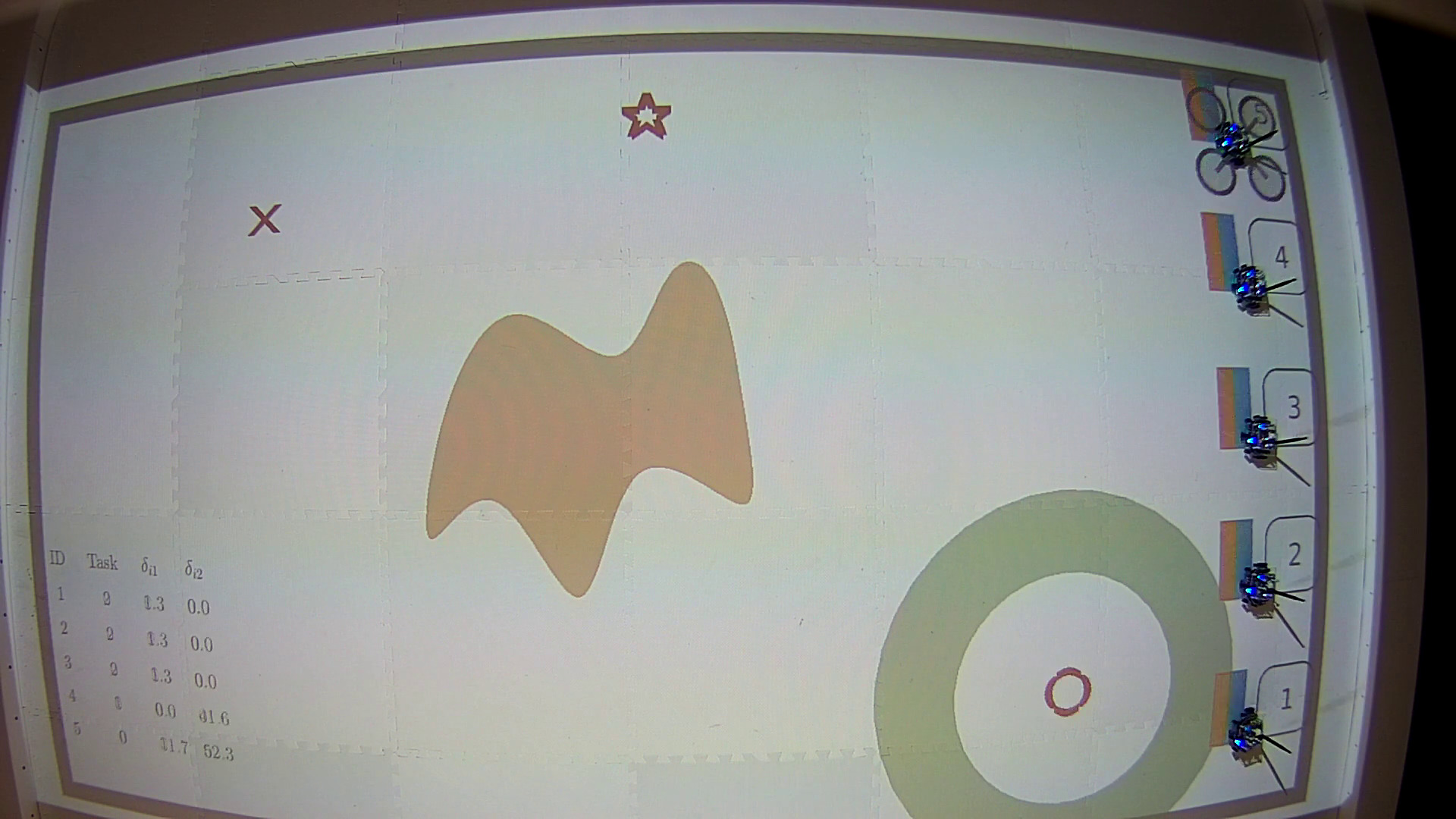}}\hfill
\subfloat[][]{\label{subfig:exp:2}\includegraphics[trim={0cm 0cm 5cm 1cm},clip,width=0.245\textwidth]{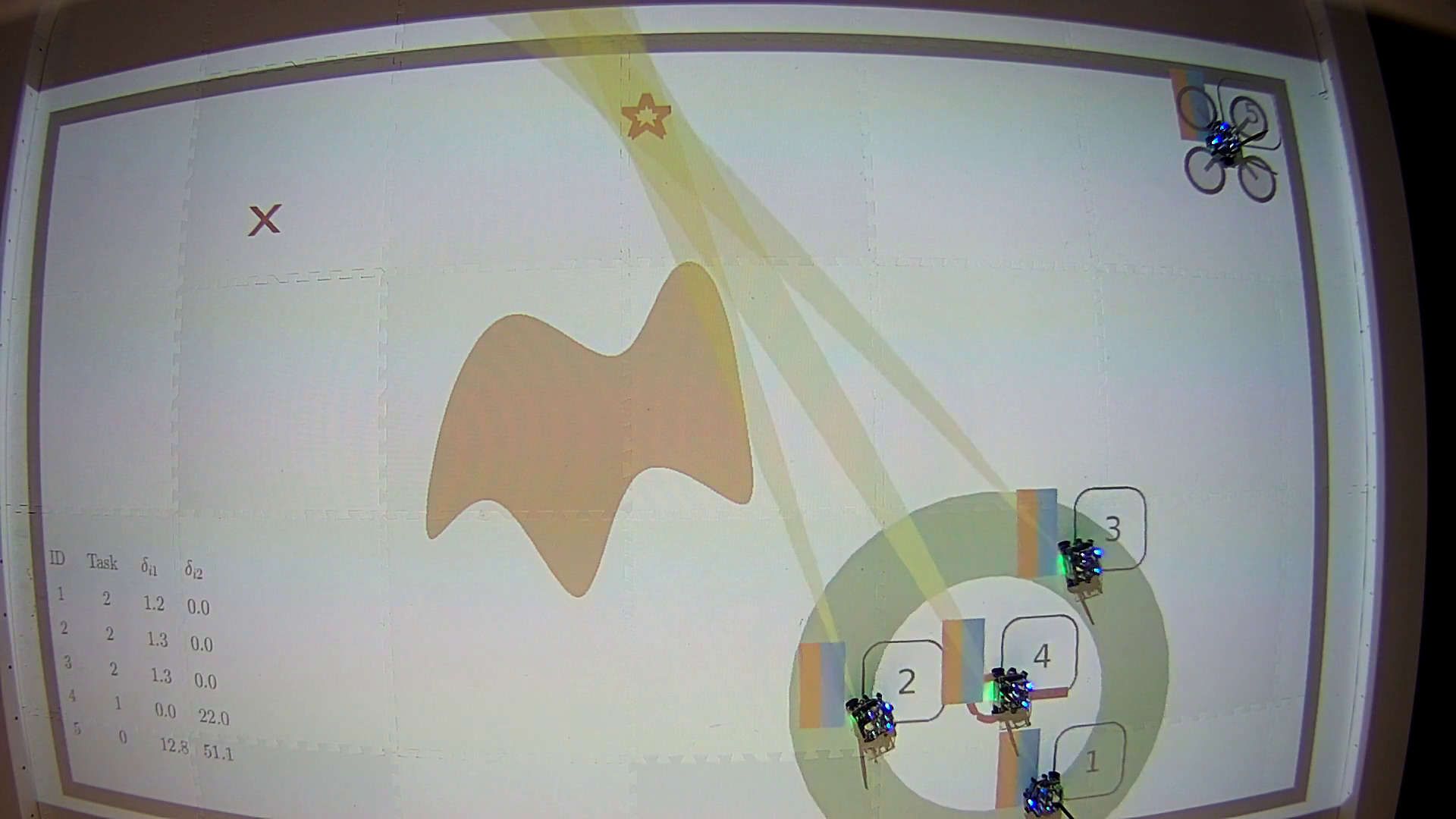}}\hfill
\subfloat[][]{\label{subfig:exp:3}\includegraphics[trim={0cm 0cm 5cm 1cm},clip,width=0.245\textwidth]{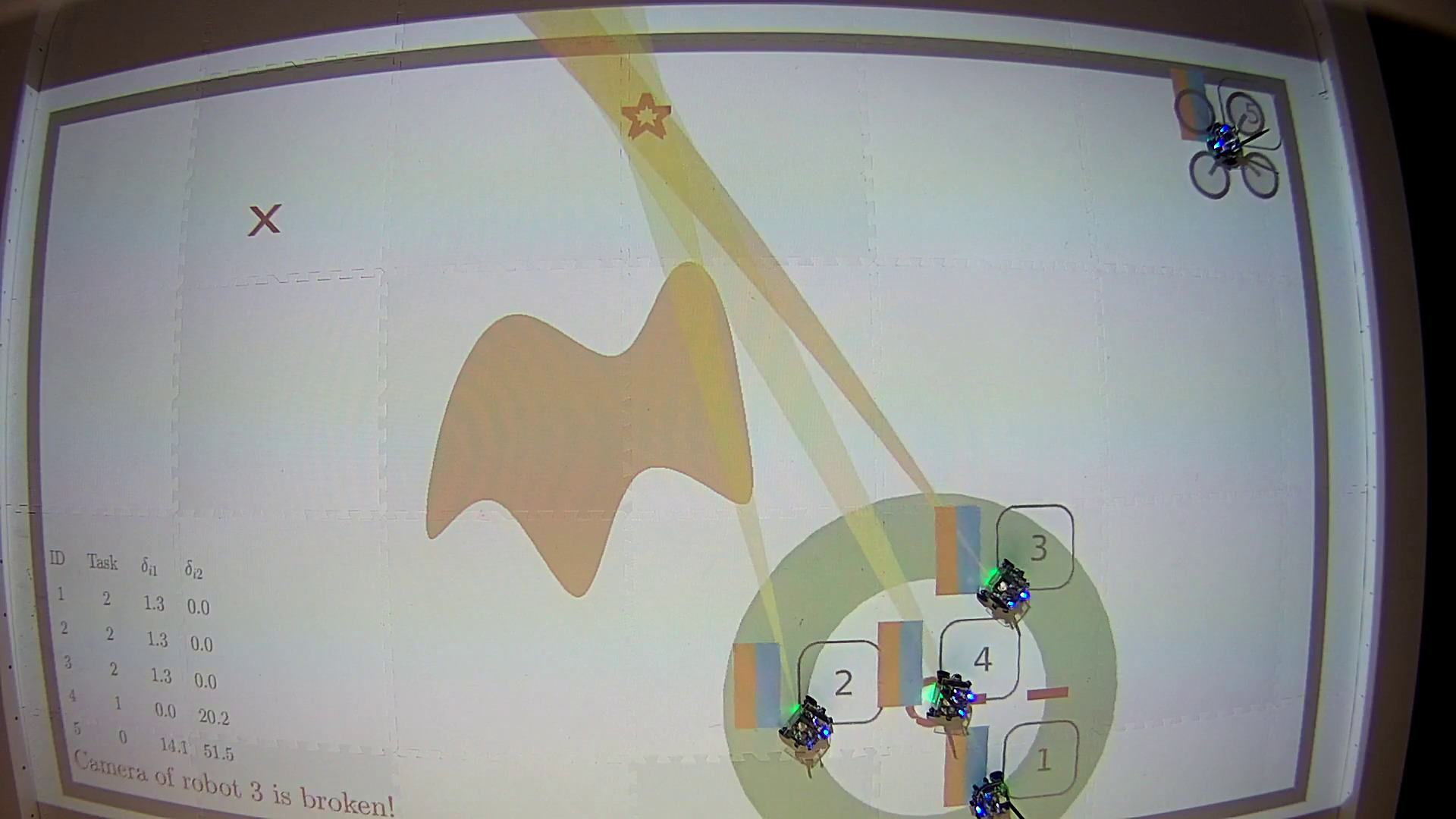}}\hfill
\subfloat[][]{\label{subfig:exp:4}\includegraphics[trim={0cm 0cm 5cm 1cm},clip,width=0.245\textwidth]{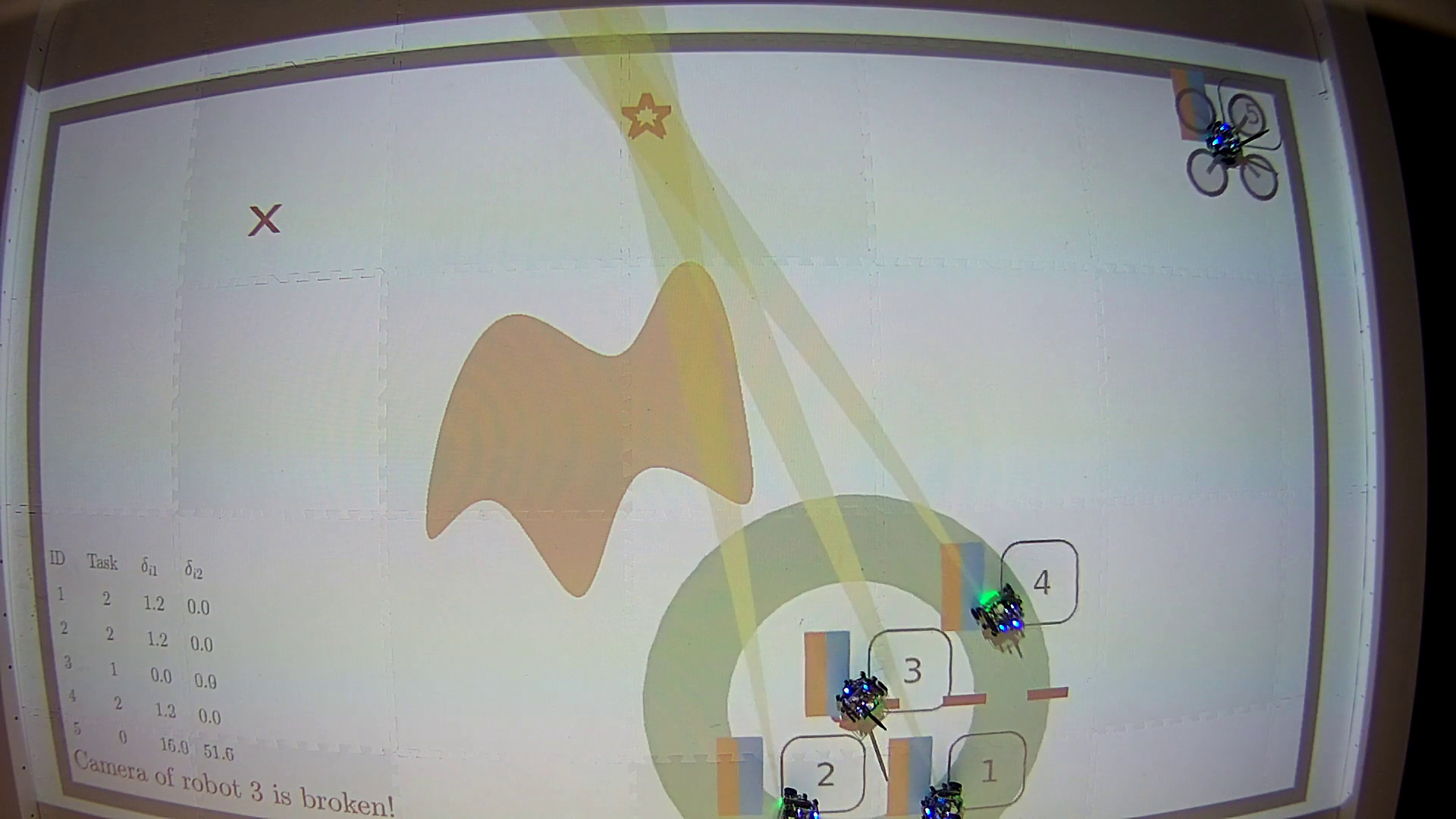}}\\
\subfloat[][]{\label{subfig:exp:5}\includegraphics[trim={0cm 0cm 5cm 1cm},clip,width=0.245\textwidth]{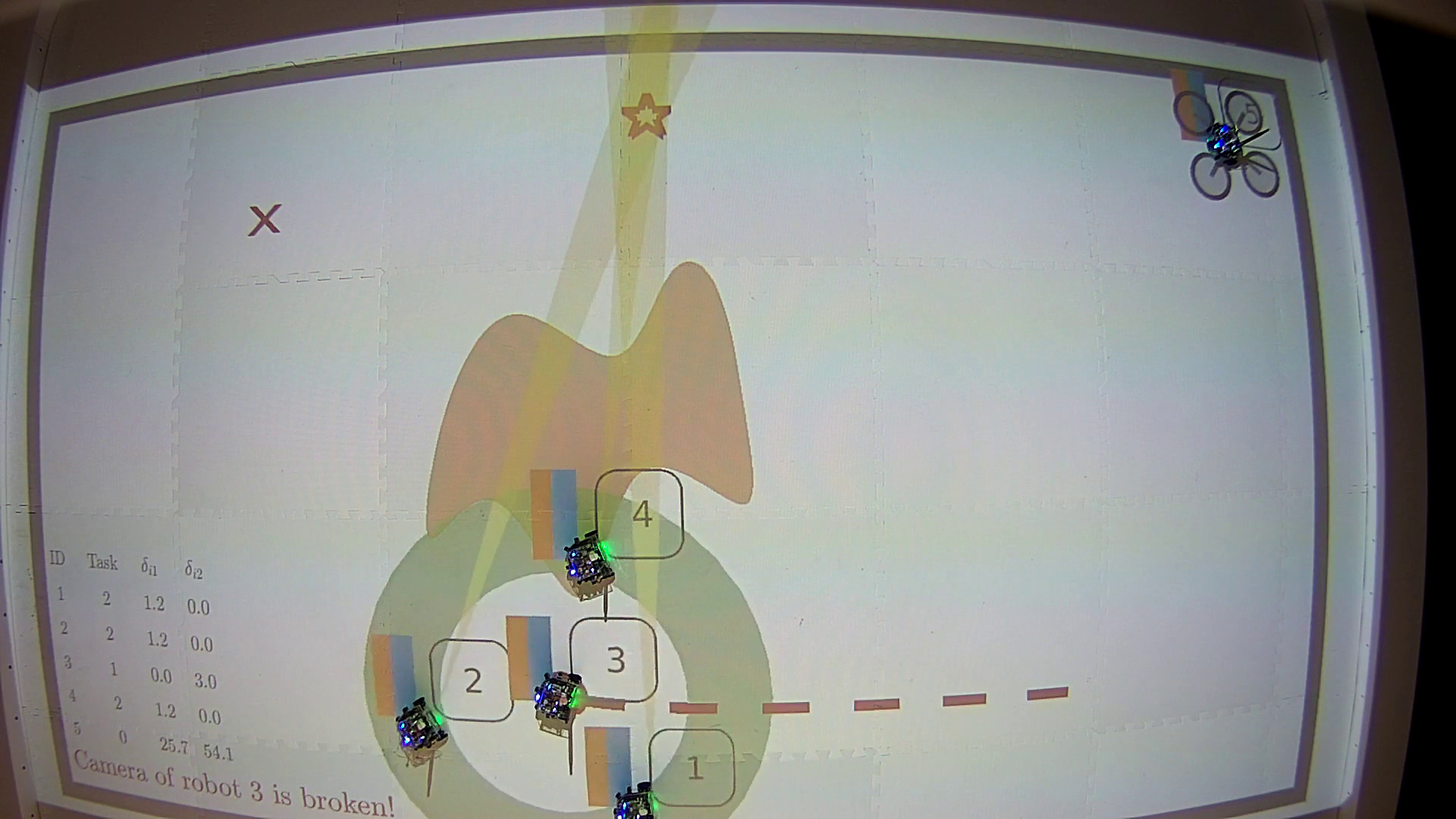}}\hfill
\subfloat[][]{\label{subfig:exp:6}\includegraphics[trim={0cm 0cm 5cm 1cm},clip,width=0.245\textwidth]{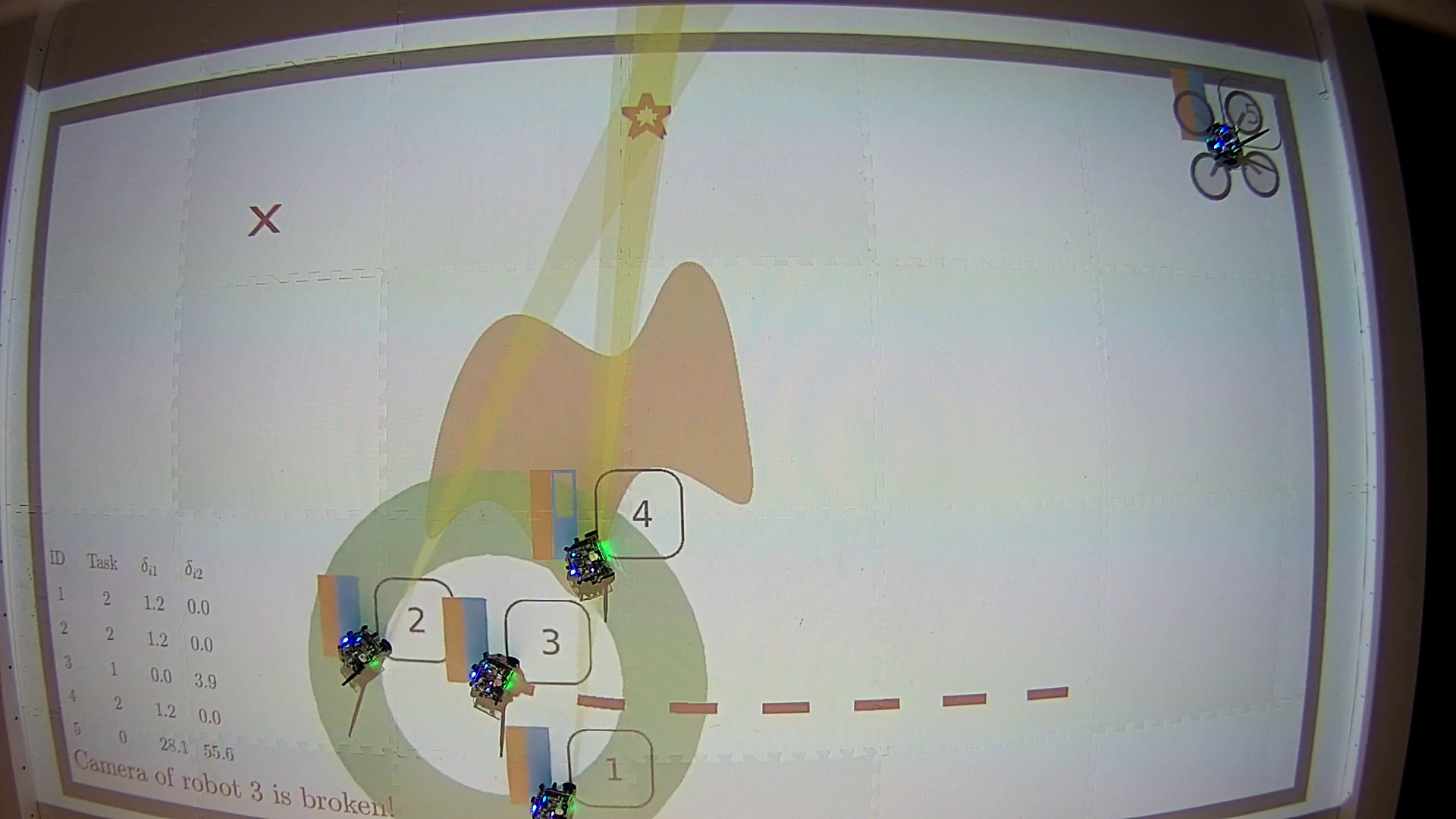}}\hfill
\subfloat[][]{\label{subfig:exp:7}\includegraphics[trim={0cm 0cm 5cm 1cm},clip,width=0.245\textwidth]{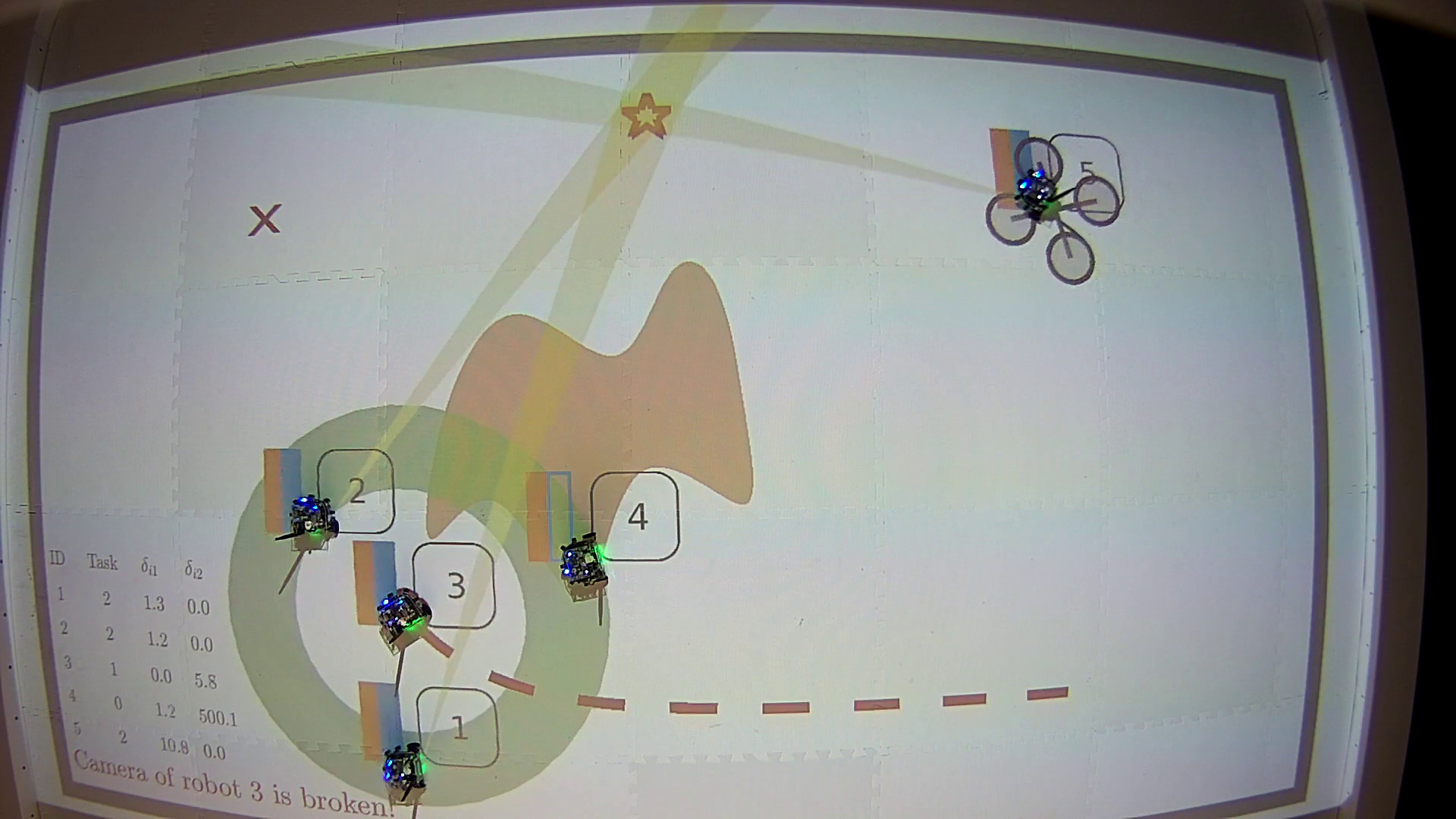}}\hfill
\subfloat[][]{\label{subfig:exp:8}\includegraphics[trim={0cm 0cm 5cm 1cm},clip,width=0.245\textwidth]{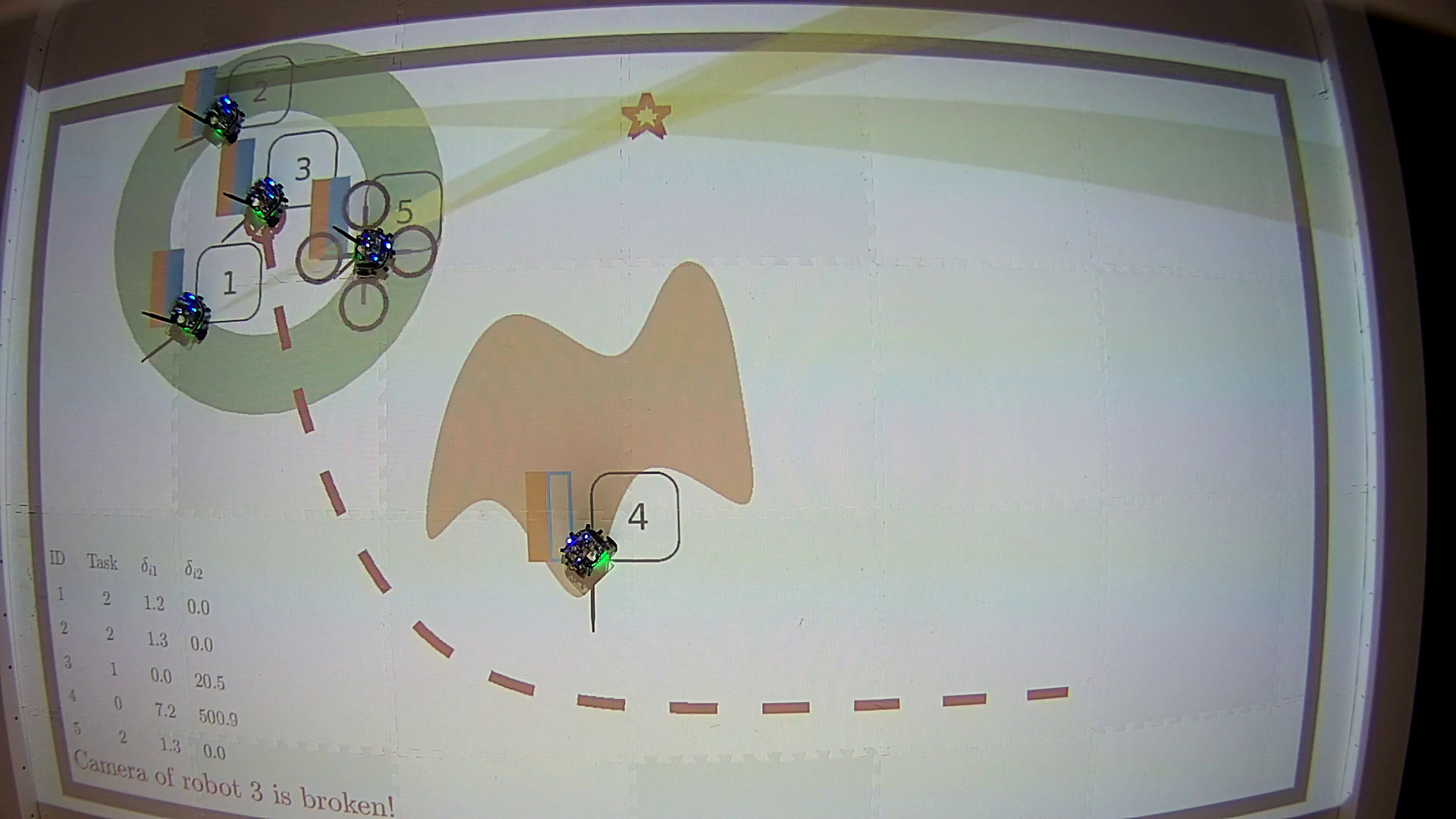}}
\caption{Snapshots recorded during the course of the experiment on the Robotarium \cite{wilson2020robotarium}. The scenario is the one depicted in Fig.~\ref{fig:experiment:scenario}. Five robots are initially arranged along the right side of the rectangular environment (Fig.~\ref{subfig:exp:1}). The robot ID is projected down onto the Robotarium testbed at the top right corner of each robot. The capabilities of each robot to perform task $t_1$ and $t_2$ are indicated by vertical progress bars at the top left corner of each robot, using the same color codes as in Fig.~\ref{fig:experiment:map}, i.e. orange for $t_1$ and blue for $t_2$. The height of the progress bars indicates the capability of the robots during the course of the experiment. Solving the task allocation optimization problem \eqref{eq:allocationalgorithmactual} results in the following initial allocation: $r_4$ is allocated to $t_1$, which entails navigating the environment to reach the red cross, and $r_1$, $r_2$ and $r_3$ are allocated to $t_3$, for which they need to escort $r_4$ during its mission. In Fig.~\protect\ref{subfig:exp:2}, $r_1$, $r_2$ and $r_3$ are arranged around $r_4$ and have pointed their cameras at the red star: the field of view of the cameras are depicted as thin yellow triangles. In Fig.~\protect\ref{subfig:exp:3}, the endogenous disturbance takes place: the camera of $r_3$ breaks---this event is represented by the the field of view of its camera becoming red. As $r_3$ is not able to perform the monitoring required for $t_2$, the constraints \eqref{eq:miqp:e} and \eqref{eq:miqp:f} result in $r_3$ swapping its allocation with $r_4$ (Fig.~\protect\ref{subfig:exp:4}). In Fig.~\protect\ref{subfig:exp:5}, the exogenous disturbance is encountered: $r_4$ is not able to move away from the simulated low-friction zone (brown area in the middle of the environment). Thanks to the update law \eqref{eq:spUpdate1}, the specialization of $r_4$ to perform $t_2$ drops to 0 (in Fig.~\protect\ref{subfig:exp:6}, the blue progress bar corresponding to task $t_2$ next to $r_4$ is emptying). The task allocation recruits $r_5$ to perform $t_2$, while $r_4$ is not assigned to any task (Fig.~\protect\ref{subfig:exp:7}). In Fig.~\protect\ref{subfig:exp:8}, the robot team has successfully completed both tasks as desired: 1 robot has reached the goal point (red cross) while being escorted at all times by 3 more robots. The full video of the experiment is available  online at \texttt{https://youtu.be/fdfYID7u72o}, where it is also possible to see, at the bottom left of the frames, a table containing current allocated task and values of the components of $\delta_i$ for each robot over the course of the experiment.}
\label{fig:experiments}
\end{figure*}

The total duration of the experiment is $80$ seconds. During this time span, the resilience of the allocation algorithm to both endogenous and exogenous disturbances is tested. At time $t=15$~s, the feature $f_3$ of robot $r_3$ is lost (endogenous disturbance), depicted by the dashed red edge on the hypergraph in Fig.~\ref{fig:experiment:map}. Moreover, in the middle of the environment, a region of low friction is present (brown blob in Fig.~\ref{fig:experiment:scenario}). This prevents the robots endowed with wheels from moving (exogenous disturbance).

In Fig.~\ref{fig:experiments}, snapshots recorded during the course of the experiment are shown. The robots start on the right of the rectangular environment (Fig.~\ref{subfig:exp:1}). The task prioritization and execution framework results in the following allocation: $r_4$ is allocated to $t_1$ and therefore has to navigate the environment to reach the red cross, while $r_1$, $r_2$ and $r_3$ are assigned to $t_3$ and thus need to escort $r_4$ during its mission. Using coverage control, they arrange themselves around $r_4$ and point their cameras---whose field of view is depicted through a yellow beam projected down onto the Robotarium testbed---at the red star (Fig.~\ref{subfig:exp:2}). At $t=15$~s, the camera of $r_3$ breaks (Fig.~\ref{subfig:exp:3}). Therefore, it cannot keep on executing $t_2$. The constraints \eqref{eq:miqp:e} and \eqref{eq:miqp:f} result in $r_3$ swapping its allocation with $r_4$ (Fig.~\ref{subfig:exp:4}). Around $t=50$~s, one of the robots, specifically $r_4$, encounters the low-friction zone, and, as a result, its motion is impeded (Fig.~\ref{subfig:exp:5}). The update law \eqref{eq:spUpdate1} makes the specialization of $r_4$ towards task $t_2$ drop (depicted as progress bars next to $r_4$ in Fig.~\ref{subfig:exp:6}). When the specialization of $r_4$ towards task $t_2$ reaches 0, the task allocation driven by the cost in \eqref{eq:allocationalgorithmactual} changes once again to adapt to the unexpected environmental conditions: $r_5$ is recruited to perform $t_2$ while $r_4$ is relieved of its duty (Fig.~\ref{subfig:exp:7}). The last snapshot (Fig.~\ref{subfig:exp:8}) shows the robot team successfully accomplishing both tasks as desired: $1$ robot has reached the goal point (red cross) while being escorted at all times by $3$ more robots.

\begin{figure}
\centering
\includegraphics[width=0.42\textwidth]{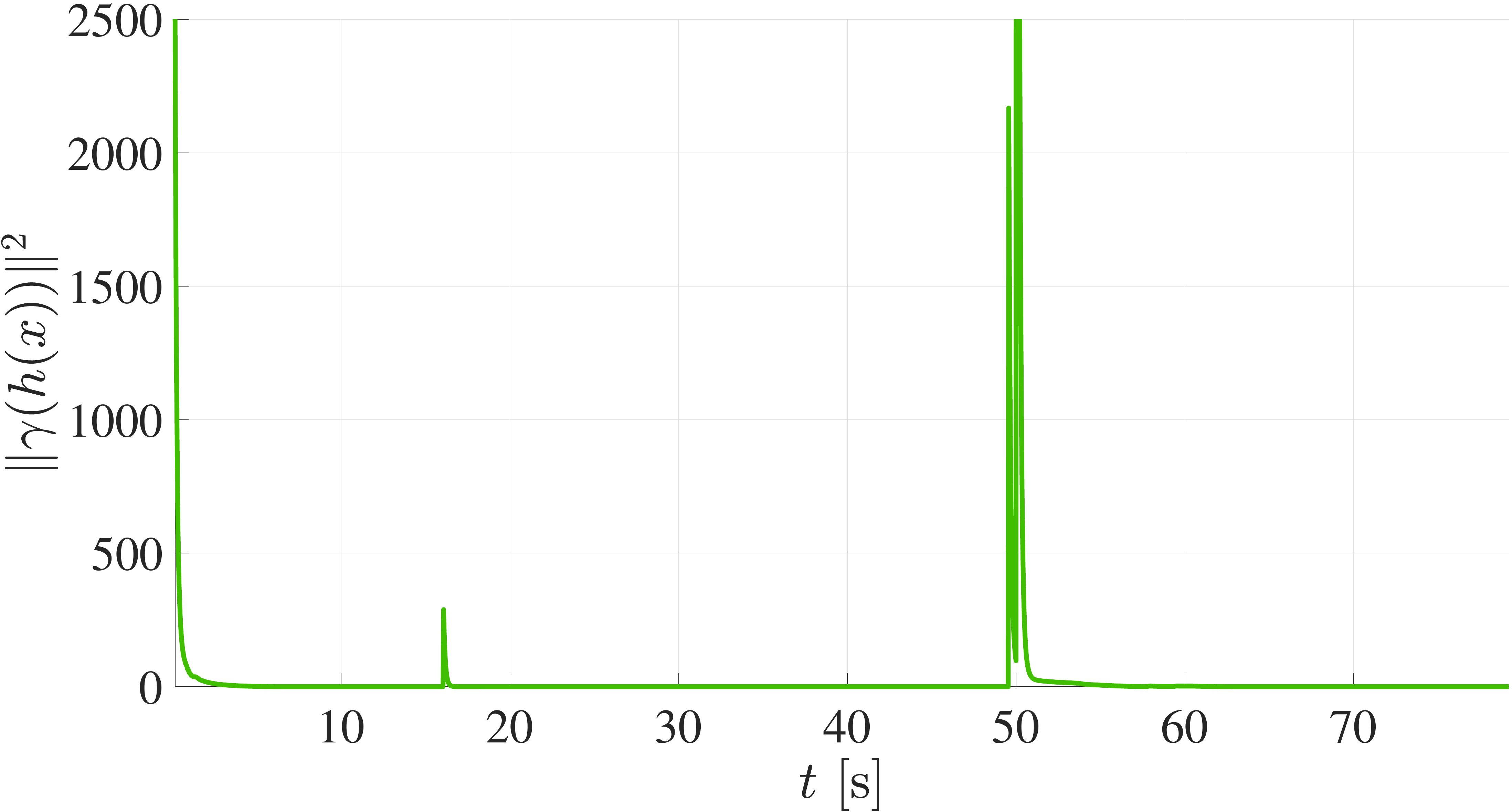}
\caption{Trajectory of the value of the Lyapunov function \eqref{eq:lyapunov} recorded over the course of the Robotarium experiments. At the beginning of the experiment, it decreases as the robots perform the assigned tasks. The endogenous and exogenous disturbances at $t=15$~s and $t=50$~s, respectively, make the value of the Lyapunov function jump to higher values, which are promptly decreased by the execution of the tasks by the robots, owing to the stability guarantees given in Proposition~\ref{prop:general}.}
\label{fig:lyapunov}
\end{figure}

The result of Proposition~\ref{prop:general} gives us another way of highlighting the resilience of the task allocation algorithm, by observing the trajectory of the Lyapunov function \eqref{eq:lyapunov}. Its value recorded over the course of the experiment is depicted in Fig.~\ref{fig:lyapunov}. At the beginning of the experiment, the value of the Lyapunov function $V$ decreases as the robots perform the assigned tasks. The endogenous disturbance at $t=15$~s makes the allocation swap: by means of the stability properties highlighted in Proposition~\ref{prop:general}, the allocation algorithm makes the robots perform forward progress towards the accomplishment of the tasks---which results in a decrease of the Lyapunov function for $t>15$~s. Towards the end of the experiment, a similar situation is observed where the exogenous disturbance of one of the robots incapable of moving anymore results in a change of the task allocation. Again, owing to the aforementioned stability properties, the execution of the tasks makes the Lyapunov function decrease again towards $0$ after a jump due to the allocation swap.

\begin{figure}
\centering
\includegraphics[width=0.42\textwidth]{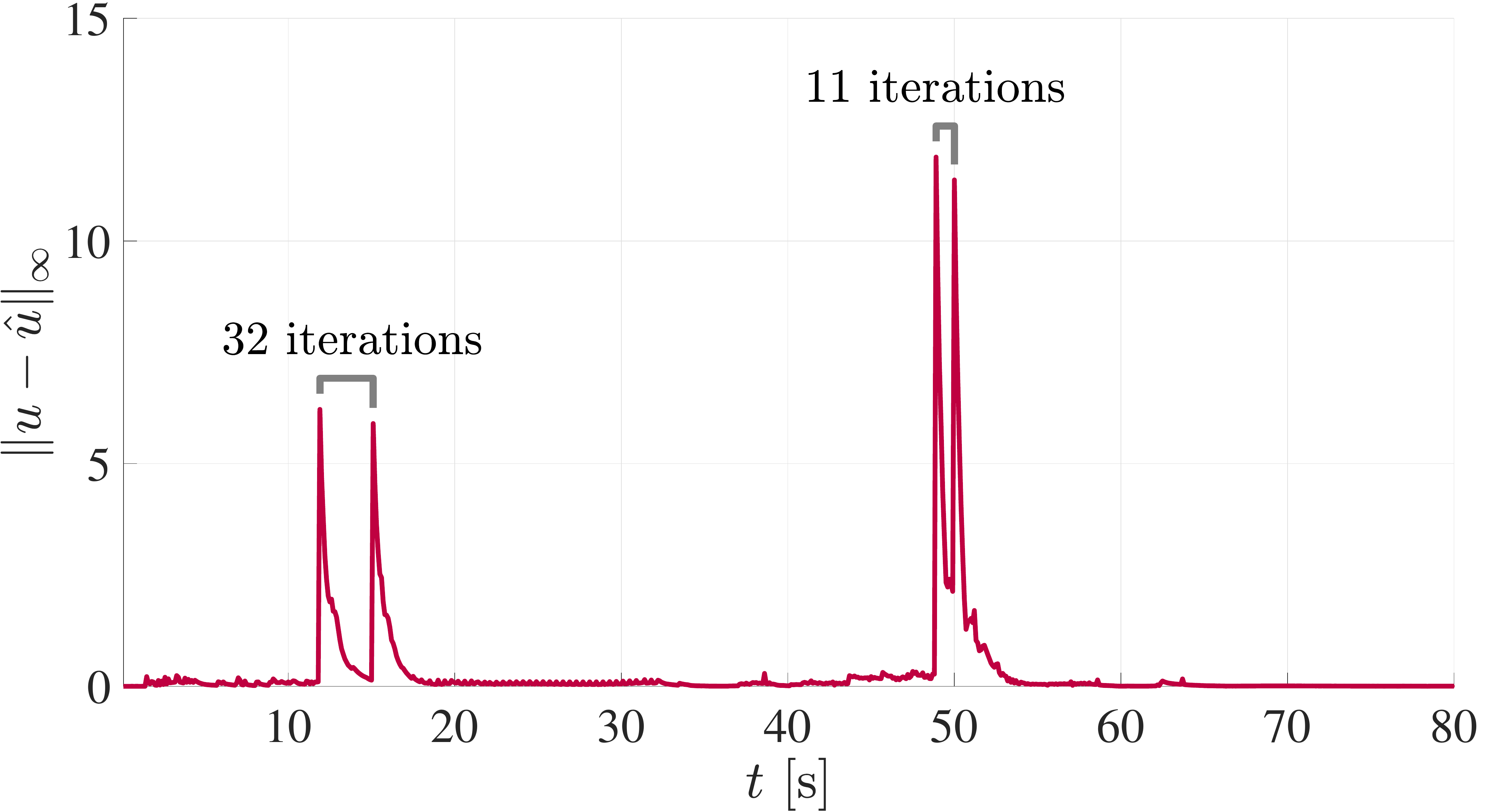}
\caption{Comparison, in terms of robot input difference, between simulations of mixed centralized/decentralized task allocation \eqref{eq:executionalgorithm} with up-to-date and outdated $\alpha$, respectively. The former has been obtained by solving the MIQP \eqref{eq:allocationalgorithmactual} at each time step and provide each robot with its allocation vector $\alpha_{-,i}$ in order to solve the QP \eqref{eq:executionalgorithm}. Without endogenous or exogenous disturbances, the difference between the inputs $\hat u$---obtained by solving the QP \eqref{eq:executionalgorithm} with $\alpha_{-,i}$ obtained by the MIQP \eqref{eq:allocationalgorithmactual} \textit{at each time step}---and $u$---synthesized using the QP \eqref{eq:executionalgorithm} with $\alpha_{-,i}$ obtained from the MIQP \eqref{eq:allocationalgorithmactual} \textit{whenever it is available}---is close to $0$. The difference peaks around the times of the endogenous and exogenous disturbances: this phenomenon is due to the delay introduced by the time required to solve the MIQP \eqref{eq:allocationalgorithmactual}. The allocation changes 34 and 11 iterations later in the case of endogenous and exogenous disturbances, respectively. This effect due to the computation time is highlighted in \eqref{eq:udiff}.}
\label{fig:mixed_comparison}
\end{figure}

To conclude, as observed in Section~\ref{sec:mixedc-dec}, the developed task prioritization and execution framework would not be realizable in realistic scenarios unless a mixed centralized/decentralized strategy is implemented. In the Robotarium experiment, two communicating processes have run in parallel: one responsible for solving the task allocation optimization problem \eqref{eq:allocationalgorithmactual}, and one with the objective of synthesizing the controller for the robots given the task allocation, using \eqref{eq:executionalgorithm}. To show the difference between the implementation of a purely centralized allocation strategy versus a mixed centralized/decentralized one, both have been simulated and the results in terms of difference between robot inputs are reported in Fig.~\ref{fig:mixed_comparison}. From the graph, it is clear that without the effect of disturbance, the difference between the inputs $\hat u$---obtained by solving the QP \eqref{eq:executionalgorithm} with $\alpha_{-,i}$ obtained by the MIQP \eqref{eq:allocationalgorithmactual} \textit{at each time step}---and $u$---synthesized using the QP \eqref{eq:executionalgorithm} with $\alpha_{-,i}$ obtained from the MIQP \eqref{eq:allocationalgorithmactual} \textit{whenever it is available}---is close to $0$. The peaks around the times of the endogenous and exogenous disturbances are due to the fact that in the mixed centralized/decentralized case, there is a delay of $n$ times steps (the effect of computation time in \eqref{eq:udiff}) in recomputing the task allocation. In fact, solving the MIQP \eqref{eq:allocationalgorithmactual} takes on average 100 steps required to solve the QP \eqref{eq:executionalgorithm}, using the MATLAB CVX library \cite{grant2009cvx} and the Gurobi solver \cite{optimization2014inc}.

\section{Conclusion} \label{sec:conc}

In this paper, we have presented an optimization-based task prioritization and execution framework that achieves a resilient and energy-aware task allocation strategy for heterogeneous multi-robot systems. The approach lies its foundations on a proposed decomposition of the ability of the robots at performing tasks into features, capabilities and specialization of the robots. Moreover, the approach builds up on the notion of set-based tasks, where each task executed by the robots is characterized by a set encoded using a control barrier function. These modeling choices allow us to prioritize tasks by considering the different specialization that different robots have at performing different tasks, effectively realizing a \emph{heterogeneous} task allocation. Furthermore, the optimization-based and pointwise-in-time nature of the task allocation algorithm contribute to foster its \emph{resilience} properties.

We showed ways to achieve resilience with respect to endogenous disturbances (failure of a robot caused by loss of features) as well as exogenous disturbances (caused by unmodeled phenomena in the environment) which leverage the reactive nature of the formulation. Moreover, we demonstrated how the formulation allows us to specify both the number of robots and the amount of capabilities required to perform a certain task. This way, thanks to the energy-awareness of the algorithm, robots which are not required to perform tasks are not utilized. Nevertheless, they can potentially be recruited at any point in time, achieving, this way, autonomy-on-demand in the context of task allocation.

The effectiveness of the proposed approach is showcased through a mixed centralized/decentralized implementation of the developed task allocation strategy on a team of 5 mobile robots, possessing 3 features and 2 capabilities to perform 2 tasks, under both endogenous and exogenous disturbances.

\onecolumn
\appendices

\section{Matrices Defined in Propositions~\ref{prop:general}}
\label{app:eq:F}

\begin{equation}
	\label{eq:F}
	\begin{gathered}
		B_0^{(k)} = \begin{bmatrix}
			cI & \frac{\mathrm{d}\gamma}{\mathrm{d}h}\frac{\mathrm{d}h}{\mathrm{d}x} g(x^{(k)}) & 0 & 0 & \frac{\mathrm{d}\gamma}{\mathrm{d}h}\frac{\mathrm{d}h}{\mathrm{d}x} f(x^{(k)})\\
			\frac{\mathrm{d}\gamma}{\mathrm{d}h}\frac{\mathrm{d}h}{\mathrm{d}x}g(x^{(k)})\tr & 0 & 0 & 0 & 0\\
			0 & 0 & 0 & 0 & 0\\
			0 & 0 & 0 & 0 & 0\\
			\frac{\mathrm{d}\gamma}{\mathrm{d}h}\frac{\mathrm{d}h}{\mathrm{d}x} f(x^{(k)})\tr & 0 & 0 & 0 & 0\\
		\end{bmatrix},\\
		B_1^{(k)} = \begin{bmatrix}
			0 & 0 & -\frac{1}{2}I & 0 & 0\\
			0 & 0 & -\frac{1}{2}L_gh(x^{(k)})\tr & 0 & 0\\
			-\frac{1}{2}I & -\frac{1}{2}L_gh(x^{(k)}) & -I & 0 & -\frac{1}{2}L_fh(x^{(k)})\\
			0 & 0 & 0 & 0 & 0\\
			0 & -\frac{1}{2}L_fh(x^{(k)})\tr & 0 & 0 & 0\\
		\end{bmatrix},\\
		B_2^{(k)} = \begin{bmatrix}
			0 & 0 & 0 & 0 & 0\\
			0 & 0 & 0 & 0 & 0\\
			0 & 0 & 0 & \frac{1}{2}\bar\Theta\bar\Phi & 0\\
			0 & 0 & \frac{1}{2}\bar\Phi\tr\bar\Theta\tr & \bar\Phi\tr\bar\Phi & \frac{1}{2}\bar\Phi\bar\Psi\\
			0 & 0 & 0 & \frac{1}{2}\bar\Psi\tr\bar\Phi\tr & 0\\
		\end{bmatrix},\qquad
		B_3^{(k)} = \begin{bmatrix}
			0 & 0 & 0 & 0 & 0\\
			0 & 0 & 0 & 0 & 0\\
			0 & 0 & A_\delta\tr A_\delta & 0 & \frac{1}{2}A_\delta\tr b_\delta\\
			0 & 0 & 0 & A_\alpha\tr A_\alpha & \frac{1}{2}A_\alpha\tr b_\alpha\\
			0 & 0 & \frac{1}{2}b_\delta\tr A_\delta & \frac{1}{2}b_\alpha\tr A_\alpha & 0
		\end{bmatrix}
	\end{gathered}
\end{equation}

\section{Bounds Employed in Subsection~\ref{subsec:mixed}}
\label{app:eq:udiff}

\begin{equation}
	\label{eq:difference}
	\begin{aligned}
		\left\|\ukn-\hatukn\right\|_\infty &= \left\|\Gamma\left(\xkn,\xk\right)-\Gamma\left(\xkn,\xkn\right)\right\|_\infty\\
		&= \left\|\GammaQP\left(\xkn,\GammaMIQP\left(\xk\right)\right)-\GammaQP\left(\xkn,\GammaMIQP\left(\xkn\right)\right)\right\|_\infty
	\end{aligned}
\end{equation}

\begin{equation}
	\label{eq:udiff}
	\begin{aligned}
		\left\|\ukn-\hatukn\right\|_\infty &= \left\|\GammaQP\left(\xkn,\GammaMIQP\left(\xk\right)\right)-\GammaQP\left(\xkn,\GammaMIQP\left(\xkn\right)\right)\right\|_\infty\\
		&\le \left\|\GammaQP\left(\xkn,\GammaMIQP\left(\xk\right)\right)-\GammaQP\left(\xkn,\GammabarMIQP\left(\xk\right)\right)\right\|_\infty\\
		&\quad+\left\|\GammaQP\left(\xkn,\GammabarMIQP\left(\xk\right)\right)-\GammaQP\left(\xkn,\GammabarMIQP\left(\xkn\right)\right)\right\|_\infty\\
		&\quad+\left\|\GammaQP\left(\xkn,\GammaMIQP\left(\xkn\right)\right)-\GammaQP\left(\xkn,\GammabarMIQP\left(\xkn\right)\right)\right\|_\infty\\
		&\le L_\text{QP}\Big(\left\|\left(\xkn,\GammaMIQP\left(\xk\right)\right)-\left(\xkn,\GammabarMIQP\left(\xk\right)\right)\right\|_\infty\\
		&\quad+\left\|\left(\xkn,\GammabarMIQP\left(\xk\right)\right)-\left(\xkn,\GammabarMIQP\left(\xkn\right)\right)\right\|_\infty\\
		&\quad+\left\|\left(\xkn,\GammaMIQP\left(\xkn\right)\right)-\left(\xkn,\GammabarMIQP\left(\xkn\right)\right)\right\|_\infty\Big)\\
		&\le L_\text{QP}\Big(n_t^2 n_r^3 m \Delta\left(\mathscr A\left(\xk\right)\Big|\mathscr B\left(\xk\right)\right)\\
		&\quad+L_\text{MIQP}\left\|\left(\xkn,\xk\right)-\left(\xkn,\xkn\right)\right\|_\infty\\
		&\quad+n_t^2 n_r^3 m \Delta\left(\mathscr A\left(\xkn\right)\Big|\mathscr B\left(\xkn\right)\right)\Big)\\
		&\le L_\text{QP}\Big(n_t^2 n_r^3 m \Delta\left(\mathscr A\left(\xk\right)\Big|\mathscr B\left(\xk\right)\right)\\
		&\quad+L_\text{MIQP}\left(\left\|\x{k}-\x{k+1}\right\|_\infty+\left\|\x{k+1}-\x{k+2}\right\|_\infty+\ldots+\left\|\x{k+n-1}-\x{k+n}\right\|_\infty\right)\\
		&\quad+n_t^2 n_r^3 m \Delta\left(\mathscr A\left(\xkn\right)\Big|\mathscr B\left(\xkn\right)\right)\Big)\\
		&\le \underbrace{L_\text{QP}n_t^2 n_r^3 m \left(\Delta\left(\mathscr A\left(\xk\right)\Big|\mathscr B\left(\xk\right)\right)+\Delta\left(\mathscr A\left(\xkn\right)\Big|\mathscr B\left(\xkn\right)\right)\right)}_\text{Effect of mixed-integer programming}+\underbrace{L_\text{QP}L_\text{MIQP} L_{\dot x} n \Deltat}_\text{Effect of computation time}.
	\end{aligned}
\end{equation}

\twocolumn

\bibliographystyle{IEEEtran}
\bibliography{bib/IEEEabrv,bib/references}

\end{document}